\documentclass{article}
\usepackage{microtype}
\usepackage{graphicx}
\usepackage{subcaption}
\usepackage{booktabs} %
\usepackage{hyperref}
\usepackage{enumitem}
\usepackage{amsmath}
\usepackage{amssymb}
\usepackage{mathtools}
\usepackage{amsthm}
\usepackage{bbm, dsfont}
\usepackage{multirow}
\usepackage{colortbl}
\usepackage{algorithm}
\usepackage{algorithmic}
\usepackage[capitalize]{cleveref}
\usepackage{aliascnt}
\PassOptionsToPackage{numbers, compress}{natbib}

\usepackage[preprint]{neurips_2026}

\newcommand{\Edges}{E}

\newcommand{\Espace}{\Omega}

\newcommand{\indicator}[1]{\mathbbm{1}_{\left\{#1\right\}}}

\newcommand{\zols}{\bar{z}}
\newcommand{\Bols}{\bar{B}}

\newcommand{\thetaols}{\bar{\theta}}

\newcommand{\zhat}{\hat z}
\newcommand{\pihat}{\hat{\pi}}
\newcommand{\Bhat}{\operatorname{B}\left(\zhat\right)}
\newcommand{\nuhat}{\nu\left(\zhat\right)}
\newcommand{\thetahat}{\hat{\theta}}

\newcommand{\ztrue}{z^*}
\newcommand{\Btrue}{B^*}
\newcommand{\nutrue}{\nu^*}
\newcommand{\thetatrue}{\theta^*}

\DeclareMathOperator*{\argmax}{arg\,max}
\DeclareMathOperator*{\argmin}{arg\,min}

\newcommand{\cT}{\mathcal{T}}

\newcommand{\cL}{\mathcal{L}}

\newcommand{\cJ}{\mathcal{J}}

\newcommand{\cZ}{\mathcal{Z}}

\newcommand{\bproj}[1]{\operatorname{B}\left(#1\right)}

\newcommand{\numgraphons}{6 }
\newcommand{\numsbms}{4 }
\newcommand{\totalsynthexp}{10 }

\newcommand{\sharedthm}[3]{%
  \newaliascnt{#1}{theorem}%
  \newtheorem{#1}[#1]{#2}%
  \aliascntresetthe{#1}%
  \crefname{#1}{#2}{#3}%
  \Crefname{#1}{#2}{#3}%
}

\theoremstyle{plain}
\newtheorem{theorem}{Theorem}[section]
\sharedthm{proposition}{Proposition}{Propositions}
\sharedthm{lemma}{Lemma}{Lemmas}
\sharedthm{corollary}{Corollary}{Corollaries}
\theoremstyle{definition}
\sharedthm{definition}{Definition}{Definitions}
\sharedthm{assumption}{Assumption}{Assumptions}
\theoremstyle{remark}
\sharedthm{remark}{Remark}{Remarks}

\usepackage[textsize=tiny]{todonotes}

\title{Network Learning with Semi-relaxed Gromov-Wasserstein
}

\author{%
  Charles Dufour \And
  Ulysse Naepels\\  EPFL, Institute of Mathematics \\ Lausanne, Switzerland \And
  Leonardo V.~Santoro
}

\begin{document}

\maketitle

\begin{abstract}
Estimating the generative mechanism of large-scale networks is a fundamental challenge in statistical machine learning. 
It requires the identification of the latent connectivity structure, which is in general an NP-hard combinatorial problem due to the absence of canonical node labels.
We address this challenge by allowing for probabilistic couplings, thereby relaxing the assignment problem. Our estimation framework can be formulated as a semi-relaxed Gromov–Wasserstein objective and provides a low-dimensional representation of the generative structure. We solve this via a block-coordinate conditional gradient algorithm. 
Despite the relaxation, the resulting solution is typically deterministic: in fact, we show that the optimality gap between the relaxed solution and the deterministic assignment vanishes at rate $O(1/n)$, where $n$ is the number of nodes.
This allows for tractable recovery of the underlying model and enables rigorous statistical analysis: we establish consistency and minimax-optimal convergence rates for both stochastic block models and Hölder-smooth graphons. Our implementation scales efficiently with $n$, as demonstrated on both synthetic and real-world datasets.
\end{abstract}

\section{Introduction}

A central challenge in the analysis of network data is the recovery of latent structure. In statistics and machine learning, this includes tasks such as identifying communities, summarizing connectivity patterns, and inferring underlying generative mechanisms from large networks of structured data.
Such problems are relevant across a wide range of applications, including the analysis of social networks \cite{traud2012social}, biological interaction networks \cite{zitnik2018modeling,agrawal2018large} such as protein-protein interactions \cite{hamilton2017inductive,gao2023hierarchical}, web recommender systems \cite{ying2018graph}, 
or infrastructure and communication systems \cite{li2018diffusion}.
In these settings, networks are often large, noisy, and unlabeled, which makes inference particularly challenging.

A fundamental difficulty common to network inference problems is the absence of a canonical node labeling. This permutation invariance introduces both conceptual and technical challenges: since model parameters are inherently identifiable only up to relabeling of the nodes, the problem is  linked to a challenging latent variable inference problem (i.e. a non-parametric regression with unknown design points \cite{gao_rateoptimal_2015a}). Any meaningful notion of estimation error must therefore compare networks \emph{up to permutations}, rather than at the level of individual node indices. This renders standard loss functions for adjacency matrices -- such as Frobenius or $\ell^2$ norms -- ill-suited to this setting. As a result, defining a suitable notion of distance and, in particular, what it means for an estimator to be close to the truth, becomes nontrivial.
These issues already arise in simple parametric models, such as stochastic block models (SBMs, \cite{holland_stochastic_1983a}) -- where nodes are assigned to latent blocks that govern their connectivity patterns --  and become even more pronounced in nonparametric settings, such as graphons \cite{lovasz_large_2012a}.
From a computational perspective, inference in these models typically requires estimating both the latent structural parameters and the assignment of nodes to communities, a high-dimensional combinatorial problem. Determining the assignment map is itself NP-hard, rendering standard estimation approaches computationally intractable.

A natural and popular approach to the lack of canonical ordering in network data is offered by optimal transport, and in particular by the Gromov–Wasserstein (GW) metric \cite{memoli_gromov_2011a}. The GW distance compares graphs by optimally aligning their pairwise relations, treating networks as metric measure spaces and assessing similarity by matching their internal geometry rather than their node identities. This framework has inspired a rich line of recent work, with applications to graph comparison, alignment, clustering, and barycenter computation \cite{peyre_gromovwasserstein_2016a, xu_gromovwasserstein_2019, vayer_fused_2020,bauer_zgromovwasserstein_2025a}.
While GW distances provide a natural, permutation-invariant notion of similarity between networks, their role in recovering generative parameters has only been recently algorithmically explored \cite{xu_learning_2021b,han_gmixup_2022a} but not theoretically clarified.  In particular, it remains unclear whether optimizing GW-based objectives can be interpreted as a statistically meaningful estimation procedure for latent network models or how the resulting transport plans relate to underlying node assignments and block structure. 
This gap motivates the present work, in which we show that relaxing hard node assignments to probabilistic couplings leads to a principled estimation framework that can be formulated as a semi-relaxed Gromov–Wasserstein barycenter problem, and that admits sharp statistical guarantees. \Cref{fig:flow} sketches the overall pipeline: from an observed adjacency matrix, we recover a low-dimensional block model that summarises the latent generative mechanism.

\paragraph{Contributions}

\begin{figure}[t]
    \centering
    \includegraphics[width=0.95\linewidth]{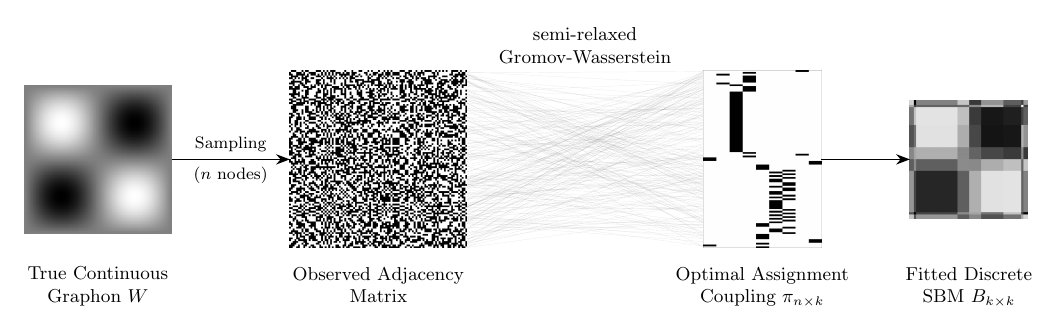}
    \caption{
    A finite adjacency matrix is sampled from a continuous latent graphon $W$. By minimizing the semi-relaxed Gromov-Wasserstein distance, we identify an optimal assignment coupling $\pi_{n \times k}$—observed here to be deterministic—which maps nodes to latent blocks. This results in a fitted discrete SBM $B_{k \times k}$ that serves as a low-dimensional approximation of $W$.}
    \label{fig:flow}
\end{figure}

In this work, we address the problem of estimating latent network structure from a single observed adjacency matrix, framing it as a semi-relaxed Gromov-Wasserstein (srGW) barycenter problem. Our theoretical analysis is developed for simple (binary undirected) graphs, though 
the method can be extended to handle (multiple) decorated graphs (i.e., with edge features), as shown in our real-data applications.

We define our estimate for the latent generating mechanism as the minimizer of a srGW objective between the observed adjacency matrix ($n\times n$) and a latent block model ($k\times k$), which discretizes the underlying truth; this can be interpreted as a form of nonparametric dimensionality reduction for network data. 
This formulation naturally yields an optimal transport plan $\pi$ that defines a probabilistic alignment mapping observed nodes to latent blocks, while simultaneously inferring block weights. Because the srGW framework operates intrinsically over the space of probabilistic couplings, it bypasses the combinatorial intractability of deterministic node assignments, ensuring its computational tractability. 
A central challenge with probabilistic relaxations is ensuring the resulting soft assignments carry meaningful discrete interpretations. We provide a rigorous theoretical bridge for this and show that the optimal solutions of the semi-relaxed problem concentrate near deterministic assignments. Specifically, we bound the optimality gap between the relaxed solution and its corresponding discrete partition, showing it vanishes at a rate of $O(1/n)$. This result provides a formal justification for employing probabilistic relaxations as a tool to recover discrete assignments: the excess cost of discretizing the optimal coupling decays as the observed network grows, allowing us to obtain \emph{near-optimal} deterministic solutions (NP-hard to find) by projecting \emph{optimal} probabilistic ones (efficiently computable).

These insights are realized through a scalable block-coordinate conditional gradient algorithm that recovers latent block structures and weights efficiently. Our implementation is designed for modern computational workflows: it is GPU-accelerated, scales to networks with thousands of nodes in seconds, and can handle multiple observed networks and graphs with edge attributes.

Our main contributions can be summarized as follows:
\begin{enumerate}[itemsep=0em, topsep=0em]
\item We cast SBM and graphon estimation as a semi-relaxed Gromov--Wasserstein barycenter problem, introducing a probabilistic relaxation of node assignments.
\item We prove that relaxed solutions concentrate near deterministic assignments, establishing a finite-sample bound on the optimality gap (\Cref{prop:optimality_gap}). We show that our estimator is statistically consistent and achieves minimax-optimal rates for both SBMs and smooth graphons (\Cref{thm:estimation_1}), in both the dense and the sparse \citep{klopp_oracle_2017b} regimes (\Cref{cor:sparse}).
\item We design a practical, GPU-accelerated block-coordinate solver (\Cref{alg:BCM_relaxed_SBM}, BCM). Our implementation scales efficiently with network size and demonstrates strong performance on both synthetic and real-world networks.
\end{enumerate}
Overall, a key strength of the proposed framework -- beyond its minimax-optimal guarantees in both the dense and sparse regimes -- is its flexibility: it applies without modification to bounded edge-decorated graphs in the sense of \cite{dufour_inference_2024a}; we use this in our real-data examples (\Cref{sec:data_analysis})

\paragraph{Related Work}
\label{sec:related_work}
Inference of latent network structures has a rich history, particularly through stochastic block models and graphons, see for example \cite{olhede_network_2014b, klopp_oracle_2017b, chatterjee_matrix_2015a, amini_semidefinite_2018a, chan_consistent_2014a}. The maximum likelihood estimator \cite{celisse_consistency_2012a} is statistically optimal but comes with NP‑hard computation. Spectral methods \cite{rohe_spectral_2011, lei_consistency_2015} offer efficiency but are sensitive to degree heterogeneity.
Our work sidesteps the direct labeling problem and leverages the notion of semi-relaxed Gromov-Wasserstein barycenters.
The GW distance \cite{memoli_gromov_2011a,peyre_gromovwasserstein_2016a} provides a metric between graphs that avoids explicit node-to-node correspondence.  It has been used for graph alignment \cite{xu_gromovwasserstein_2019}, graphon estimation \cite{xu_learning_2021b}, and extended to attributed networks \cite{vayer_fused_2020, titouan_optimal_2019a, bauer_zgromovwasserstein_2025a}. The semi-relaxed GW divergence, introduced by \citet{vincent-cuaz_semirelaxed_2022} relaxes the constraint on one marginal, and was successfully applied to graph clustering, dictionary learning \cite{vincent-cuaz_online_2021a} and to non-linear dimensionality reduction problems \cite{clark_generalized_2025,assel2025distributional}. 
In a related setting, \citet{murray_probabilistic_2025} showed that optimal semi-relaxed Gromov–Wasserstein plans are typically deterministic for inner-product–based costs. We generalize their proof strategy to network inference, where costs are induced by adjacency matrices, and establish finite-sample bounds on the optimality gap between relaxed (probabilistic) and constrained (deterministic) assignments; controlling this gap allows us to extend the approach of \cite{gao_rateoptimal_2015a} and prove minimax-optimal rates for stochastic block models and smooth graphons.
 This distinguishes our work from Gromov–Wasserstein barycenter methods for graphs such as \cite{xu_learning_2021b}, which target graph \emph{averaging}, require the barycenter size to be fixed in advance, and rely on entropic regularization that promotes diffuse transport plans -- none of which target the recovery of a discrete latent generating mechanism. In contrast, our conditional-gradient solver, by adapting the barycenter size and yielding near-deterministic couplings, is precisely what allows for statistical guarantees on the latent block structure.

\paragraph{Structure of the paper.} 
 In \Cref{sec:setting}, we introduce network models, including stochastic block models and graphons, and define the semi-relaxed GW objective. \Cref{sec:main-results} presents our main theoretical results, including the relaxation strategy via srGW barycenters (\Cref{subsec:hardtosoft}); bounds on the optimality gap between relaxed and deterministic assignments (\Cref{subsec:optgap}); finite-sample convergence rates for the proposed estimator (\Cref{subsec:learning}); implementation aspects by a conditional gradient algorithm (\Cref{subsec:algo}).
\Cref{sec:data_analysis} provides a concise illustration of the method’s empirical behavior on simulated and real-world networks; additional experiments, comparisons and discussions are deferred to the appendix. 
\Cref{appendix:background} contains additional background material, \ref{appendix:proofs} the proofs, and \ref{appendix:data_analysis} extensive numerical study and comparison to existing methods.

\section{Setting and background}\label{sec:setting}
Given a metric space $\Espace$, we denote by $\Espace^{n \times n}$ the space of all graphs on $n$ nodes represented by their adjacency matrices with edges in $\Espace$. 
In the following we consider undirected graphs with no self-loops so that any adjacency matrix $\Edges \in \Espace^{n \times n}$ is symmetric with zeros on its diagonal, and the edges are binary indicators ($\Espace = \{0,1\}$), but our method extends to broader settings: see \Cref{sec:data_analysis}.

To compare network structures while accounting for their inherent non-identifiability, we leverage the \emph{Gromov-Wasserstein distance} (see \Cref{appendix:gw}), which  metric focuses on the relational structure \emph{within each space}, rather than across. Formally, given two graphs on $n,m$ nodes represented by their adjacency matrices
$\Edges \in \Espace^{n \times n}, F \in \Espace^{m \times m}$ respectively, their GW distance can be written as
\[
\mathrm{GW}_2^2(\Edges, F)
=\inf_\pi
\;\sum_{i,i'=1}^n \sum_{j,j'=1}^m 
\left| \Edges_{i,i'} - F_{j,j'} \right|^2 
\, \pi_{i,j}\,\pi_{i',j'},
\]
where $\pi$ is taken over the set of couplings whose left and right marginals are the discrete uniform distributions on $\{1,\dots,n\}$ and $\{1,\dots,m\}$, respectively. 
The GW distance thus provides a mean to compare graphs in a way that is invariant to permutations of the nodes and works for different numbers of nodes, making it well-suited for our estimation problem.

\paragraph{SBMs and graphons.}
A widely used generative model for networks is the \emph{stochastic block model} \cite{holland_stochastic_1983a}. In the SBM, each node $i$ is assigned to one of $k$ latent blocks or communities, encoded by a membership function $z: [n] \to [k]$. In this framework, the probability that an edge between nodes $i$ and $j$ is present depends only on their block memberships:
\[
\mathbb{P}(\Edges_{ij} = 1 \mid z(i), z(j)) = B_{z(i), z(j)},
\]
where $B \in [0,1]^{k \times k}$ is a symmetric matrix specifying the connection probabilities between blocks.
The SBM can be viewed as a special case of a graphon \cite{lovasz_large_2012a}, where the graphon $W$ is piecewise constant over a partition of $[0,1]^2$ into $k \times k$ blocks.
That is, $W(x, y) = B_{a, b}$, for $x \in I_a, y \in I_b$,
with $\{I_1, \ldots, I_k\}$ a partition of $[0,1]$ determining the blocks.
 Under a more general graphon model generating mechanism, each node $i$ in an observed graph $\Edges \in \Espace^{n \times n}$ is associated with an unobserved latent feature $\xi_i \in [0,1]$, drawn i.i.d., and edges are formed independently given the latent variables:
\[
\mathbb{P}(\Edges_{ij} = 1 \mid \xi_i, \xi_j) = W(\xi_i, \xi_j),
\]
where $W : [0,1]^2 \to [0,1]$ is the underlying (unknown) graphon. This framework provides a nonparametric limit for dense graphs and captures a wide range of network structures beyond classical stochastic block models \cite{bickel_nonparametric_2009a,orbanz_bayesian_2015b}.

\section{Main results}
\label{sec:main-results}
We consider the problem of estimating the generating mechanism of a network from a single observed undirected graph. From an observed adjacency matrix $\Edges$ on $n$ nodes, generated from an unknown graphon $W$, we estimate a block-connectivity matrix $\Bhat\in[0,1]^{k\times k}$ and weights $\nuhat\in\Delta^{k-1}$ parameterising a low-dimensional approximation of $W$. It is well-known that any smooth graphon can be approximated by a SBM on $k$ nodes  arbitrarily well, for $k$ sufficiently large \cite{olhede_network_2014b}.  Therefore, a common strategy is to fit a block model on at most $k = k(n) \ll n$ blocks to the observed graph $E$, for $k$ possibly diverging with $n$. 
Fitting an SBM with $k$ blocks to an observed network on $n$ nodes amounts to partitioning the nodes into the corresponding blocks, and then estimating the connectivity probabilities across blocks. A suitable partition can be found by optimizing over all assignment maps from $n$ nodes to $k$ blocks, which is an intrinsically difficult combinatorial problem.
To tackle the issue, we propose a relaxation of the problem: instead of assigning each node to a single block, we map it to a probability distribution over the blocks.
The advantage of the relaxation is that it yields a practical implementation method which remains tractable even for large graphs. 
Notably, we prove that the gap between the relaxed and combinatorial problems can be bounded for finite samples and that in the large network limit such a convex relaxation comes at virtually no cost.
We will then show that this procedure yields an asymptotically consistent estimator: see \Cref{fig:flow}.

\subsection{From hard to soft assignment}\label{subsec:hardtosoft}

\paragraph{Hard assignment.} Finding the generating mechanism of a SBM corresponds to estimating simultaneously two quantities: the block membership map $z : [n] \to [k]$ and the connection probability matrix $B \in [0,1]^{k \times k}$ \cite{gao_rateoptimal_2015a}.
This can be naturally stated as a least squares problem:
\begin{align}\label{eq:combinatorial-problem}
    \argmin_{\substack{z \in \cZ_{n,k} \\ B \in [0,1]^{k \times k}}} \frac{1}{n^2}\sum_{i,j=1}^n \|\Edges_{i,j}- B_{z(i), z(j)}\|^2_2,
    \qquad 
    \text{with} \quad  \mathcal{Z}_{n,k} = \left\{z:[n] \rightarrow [k] \right\}.
\end{align}

For a fixed assignment map $z \in \cZ_{n,k}$, the value of $B$ minimizing \eqref{eq:combinatorial-problem} is given by:
\begin{equation*}
    \bproj{z}_{a,b} =  \frac{\sum_{i,j=1}^n \Edges_{i,j} \indicator{z(i)=a, z(j)=b}}{\sum_{i,j=1}^n \indicator{z(i)=a, z(j)=b}}.
\end{equation*}
Therefore, optimizing \eqref{eq:combinatorial-problem} over both memberships and block-matrices can be reduced to optimizing solely over memberships $\argmin_{z \in \cZ_{n,k}} \: \cL(\Edges, \bproj{z})$. 
This approach theoretically achieves an estimation error which vanishes like $O_p(n^{-2\alpha/(\alpha+1)} + {\log n}/{n})$ under standard smoothness assumptions \cite{gao_rateoptimal_2015a}.
However, solving such an optimization problem is NP-hard in general \cite{abbe_exact_2015,amini_semidefinite_2018a}: it requires searching over \emph{all} possible assignments of $n$ nodes to $k$ blocks, a discrete set of size $k^n$ which renders exhaustive search computationally intractable already for moderate $n$.

\paragraph{Probabilistic alignment: a tractable alternative.}
To overcome the intractability of searching over all deterministic labelings, we propose to relax the problem by considering \emph{probabilistic labelings}. Instead of assigning each node $i$ to a single block $z(i) \in [k]$, we associate each node with a probability distribution over the blocks. Formally, we define a probabilistic labeling as a matrix $\pi \in \mathbb{R}^{n \times k}$, where $\pi_{i,j}$ represents the probability that node $i$ belongs to the $j$-th block. That is, $\pi$ is in:
\begin{align*}
    \Pi_{[n],k} = \left\{\pi \in \mathbb{R}_+^{n \times k} : \sum_{a=1}^k \pi_{i,a} = \frac{1}{n}, \quad \forall i \in [n] \right\}.
\end{align*}
This can be interpreted as the set of couplings over $[n]\times[k]$ with left marginal fixed and equal to the uniform measure over $[n]$, while the right marginal is left free; this is the semi-relaxed GW formulation of \citet{vincent-cuaz_semirelaxed_2022}.
The relaxed optimization problem now reads:
\begin{align}\label{eq:relaxed-problem}
    \begin{aligned}
    (\pihat, \Bhat) &\in \argmin_{{\pi \in \Pi_{[n], k}, B \in [0,1]^{k \times k}}} \: \cL_{\Edges}\left(\pi, B \right), \text{ where}\\
    \cL_{\Edges}\left(\pi, B \right) &:= \sum_{i,j=1}^n\sum_{a,b=1}^k \|\Edges_{i,j}- {B}_{a,b}\|^2_2 \pi_{i,a} \pi_{j,b}.
    \end{aligned}
\end{align}
Observe that, as for the combinatorial problem, when $\pi$ is fixed there is a unique probability matrix $\bproj{\pi}$ minimizing the objective \eqref{eq:relaxed-problem} given by:
\begin{equation}\label{eq:block_matrix_and_weigths}
\begin{aligned}
\bproj{\pi}_{a,b} = \sum_{i,j=1}^n \frac{\Edges_{i,j} \pi_{i,a} \pi_{j,b}}{\nu(a)\nu(b)}, \qquad  
    \text{with }  \nu(a)=\sum_{i=1}^n \pi_{i,a}, 
\end{aligned}
\end{equation}
where $\nu$ encodes the relative block size associated to $\pi$.
This allows us to rewrite \Cref{eq:relaxed-problem} as:
\begin{equation}
    \label{eq:SRGWB}
    \begin{aligned}
    {\pihat} \in \argmin_{\pi \in \Pi_{[n], k}} \: \cL_{\Edges}(\pi), \qquad \text{with}\quad
    \cL_{\Edges}(\pi) = \cL_{\Edges}(\pi, \bproj{\pi} ).
    \end{aligned}
\end{equation}
This relaxation transforms the original hard combinatorial problem \eqref{eq:combinatorial-problem} into a continuous, easier one: the search space is now a convex set, which is significantly more tractable than the discrete set of deterministic labelings. However, probabilistic couplings are less easily interpretable, rendering their statistical analysis more challenging. Furthermore, the optimization landscape remains complex: in effect, we are burdened with the task of solving a GW objective, with cost function depending on the coupling, producing an intricate structure with no inherent simple optimization procedure.

In the optimization problem (4), we seek a barycenter structure that can be interpreted as a form of dimensionality reduction in the form of a lower-complexity block model that preserves the main relational structure of the graphs. The grouping of nodes into blocks enforces information sharing across nodes and allows for estimation, even if clustering is not our primary goal\footnote{see \citet[section 3.6]{gao_rateoptimal_2015a} for a discussion between clustering and parameter estimation}.

\subsection{Optimal soft assignments are nearly deterministic}\label{subsec:optgap}
As observed by \citet{murray_probabilistic_2025}, since the right marginal of the coupling is free, optimal $\pi$ tend to concentrate on deterministic assignments: each node is typically mapped to a single block. We first study the optimal coupling for a fixed block matrix, then combine with the optimal $B$ from \eqref{eq:block_matrix_and_weigths} to obtain rates for the estimator.

Let $B \in [0,1]^{k\times k}$ be a fixed block matrix and consider the problem of finding the optimal coupling $\pi \in \Pi_{[n], k}$ for fixed block structure $B$:
\begin{equation}
    \label{eq:SRGW}
     \argmin_{\pi \in \Pi_{[n], k}} \: \cL_{\Edges,B}, \qquad \text{where} \quad \cL_{\Edges,B} = \cL_{\Edges}(\pi,B)
\end{equation}
Before targeting \eqref{eq:SRGWB}, a simpler question is to understand the nature of the optimal solutions of \eqref{eq:SRGW}. 
We extend the marginal minimization result appearing in \citet[Thm. 10]{murray_probabilistic_2025}  and show that optimal solutions to \eqref{eq:SRGW} are typically not diffuse.

\begin{lemma}\label{lem:almostmap}
    Fix $B \in [0,1]^{k\times k}$ and $\Edges\in \Espace^{n \times n}$. If $\tilde\pi$ is a local-minima of $\cL_{\Edges,B}$. Then:
\[    \tilde\pi_{i,a} > 0 \Longrightarrow a \in \argmin_{\circ \in [k]} \sum_{j=1}^n \sum_{b=1}^k \|\Edges_{i,j}- B_{\circ, b}\|^2_2 \, \tilde\pi_{j,b},
\qquad \text{for all} \quad (i, a) \in [n] \times [k].
\]
\end{lemma}
In words: at any local minimum, the blocks supporting positive mass for node $i$ are those minimising the marginal cost given the assignments of all other nodes. Two consequences follow: (i) when this argmin is unique, $\tilde\pi$ \emph{must} be deterministic; (ii) when ties occur, we can rebalance the mass at node $i$ onto a single tied block without violating the first-order condition. Carefully handling all such ties yields a deterministic labeling whose excess cost is $O(1/n)$.

\begin{definition}[Rounding map]
For $\pi \in \Pi_{[n],k}$, define the \emph{rounding map}
$\operatorname{z} : \Pi_{[n],k} \to \mathcal{Z}_{n,k}$ by
\begin{equation}\label{eq:rounding}
    \operatorname{z}(\pi)
    \;=\;
    \left(\min\!\left\{a \in \argmax_{a'\in[k]}\,\pi_{i,a'}\right\}, \: i \in [n]\right),
\end{equation}
assigning each node to its heaviest block and breaking ties by smallest index.
\end{definition}

\begin{proposition}\label{prop:optimality_gap}
    Let $B \in [0,1]^{k\times k}$, and let $\tilde \pi$ be a local minimum for \eqref{eq:SRGW}.
    Then, there exists a deterministic labeling $\tilde z = \operatorname{z}(\tilde\pi)$ such that the optimality gap satisfies:
    $$
    0 \leq \cL_{\Edges, B}(\tilde z) - \cL_{\Edges, B}(\tilde \pi) \leq n^{-1}\|{B}\|_\infty^2.
    $$
\end{proposition}
The bound is independent of $k$: the gap remains $O(1/n)$ whether $k$ is fixed or grows with $n$, which is what enables the graphon rate where $k\asymp n^{1/(\alpha\wedge 1+1)}$.

\begin{remark}[From local minima to Frank-Wolfe stationarity]\label{rmk:fw_stationary}
    The Frank-Wolfe (FW) gap at a coupling $\tilde\pi\in\Pi_{[n],k}$ %
    vanishes exactly when $\tilde\pi$ itself minimises the linear functional $\pi\mapsto\langle\nabla\cL_{\Edges,B}(\tilde\pi),\,\pi\rangle$ over $\Pi_{[n],k}$. By the product structure of $\Pi_{[n],k}$, this minimisation decouples across nodes and assigns each node to a block minimising $\cJ_{\tilde\pi}(i,\cdot)$, so $\mathrm{gap}(\tilde\pi)=0$ is exactly the per-node argmin condition $\tilde\pi_{i,a}>0\Rightarrow a\in\argmin_\circ \cJ_{\tilde\pi}(i,\circ)$ of \Cref{lem:almostmap}. Inspection of the proofs (\Cref{appendix:proofs}) shows that \Cref{lem:almostmap,prop:optimality_gap} use only this condition; the conclusions therefore apply to any FW-stationary coupling, including the output of the conditional-gradient solver of \Cref{subsec:algo} at termination, and -- via \Cref{thm:estimation_1,cor:sparse} -- to the resulting estimate $(\Bhat,\nuhat)$. Whether \Cref{alg:BCM_relaxed_SBM} further converges to a global minimum remains open \citep{vincent-cuaz_semirelaxed_2022}.
\end{remark}

\subsection{Network learning with srGW barycenters}\label{subsec:learning}

To approximate $W$ from $\Edges$ we fit a block model $(B,\nu)$ with adaptive $k$ to $\Edges$. Rather than solving the NP-hard \eqref{eq:combinatorial-problem} directly, we solve the relaxed \eqref{eq:SRGWB} and exploit \Cref{prop:optimality_gap} to extract a near-optimal deterministic assignment via the following three-step procedure:
\begin{enumerate}[itemsep=0em, topsep=0em]
    \item \textbf{Solve the relaxed problem:} 
    $\hat \pi \in \argmin \{\cL_{\Edges}(\pi): \; \pi \in \Pi_{ [n], k } \} $ (via \Cref{alg:BCM_relaxed_SBM}).
    \item \textbf{Extract a deterministic assignment:} $\hat z = \operatorname{z}(\hat \pi)$ by \cref{eq:rounding}.
    \item \textbf{Compute the block estimate:} $(\hat B, \hat\nu) = \bigl(\operatorname{B}(\hat z),\, \nu(\hat z)\bigr)$ via \cref{eq:block_matrix_and_weigths}.
\end{enumerate}

The procedure described above enjoys the following optimal rates of convergence.
\begin{theorem}\label{thm:estimation_1} Let $\Edges \in \Omega^{n\times n}$. 
\textbf{(i) SBM.}%
Suppose that $\Edges$ is generated from a stochastic block model with parameters $(\Btrue, \nutrue)$ and true block assignment $\ztrue$. Let $\left(\Bhat, \nuhat\right)$ be as above. For any $C'>0$ there exists universal $C>0$ such that with probability at least $1-\exp(-C'n\log k)$:
    \[
    \frac{1}{n^2}\sum_{i,j}\left\|\Btrue_{\ztrue(i), \ztrue(j)} - \Bhat_{\zhat(i), \zhat(j)} \right\|^2 \leq C \left(\frac{k^2}{n^2} + \frac{\log k}{n}\right).
    \]

\textbf{(ii) Graphon.} Let $W$ be an $\alpha$-Hölder-smooth graphon and $\Edges$ be generated from $W$, with $\thetatrue = \mathbb{E}[\Edges]$. Setting $k=\lceil n^{1/(\alpha \wedge 1+1)}\rceil$, for any $C'>0$ there exists universal $C>0$ such that with probability at least $1-\exp(-C'n\log n)$,
    \[
    \frac{1}{n^2}\sum_{i,j}\left\|\thetatrue_{ij} - \Bhat_{\zhat(i), \zhat(j)} \right\|^2 \leq C\left(n^{-\frac{2\alpha}{\alpha+1}}+\frac{\log n}{n}\right).
    \]
\end{theorem}
\emph{Both rates are minimax-optimal} \cite{gao_rateoptimal_2015a,klopp_oracle_2017b}. The rates are usually attained by the (NP-hard) least-squares estimator over discrete assignments, whereas \Cref{thm:estimation_1} obtains them from the minimiser of a different, continuously relaxed objective -- the srGW barycenter loss in \eqref{eq:SRGWB}. The clustering rate $\log k/n$ in (i) accounts for the unknown node structure; in (ii) the rate $n^{-2\alpha/(1+\alpha)}$ is the standard nonparametric rate for $\alpha$-Hölder functions, and $\log n/n$ dominates when $\alpha\geq 1$. Unsurprisingly, since the optimality gap in \cref{prop:optimality_gap} depends on $\|B\|_\infty^2$, we can refine the rates under sparsity.

\begin{corollary}[Sparse regime]\label{cor:sparse}
Assume the true parameter scakes as  $\thetatrue_{ij}=\rho_n W^*(\xi_i,\xi_j)$ with $W^*:[0,1]^2\to[0,1]$ and $\rho_n\!\in\!(0,1]$ dictates the sparsity level. Assuming that all blocks contain a non-vanishing fraction of the nodes, $(\Bhat,\nuhat)$ match the minimax rates established by \citet{klopp_oracle_2017b} up to an additive gap of order $\|\bproj{\hat \pi}\|_\infty^2 n^{-1}$.
\end{corollary}

\paragraph{Comparison with prior work.} Our argument adapts the marginal-minimisation template of \citet{murray_probabilistic_2025}. Two refinements are needed: (a) \Cref{lem:almostmap,prop:optimality_gap} hold for \emph{local} minima, since the conditional-gradient solver only certifies stationary points \cite{lacoste-julien_convergence_2016}; (b) a direct adaptation yields an $O(k/n)$ gap, insufficient when $k\!\asymp\! n^{1/(\alpha\wedge 1+1)}$, so \Cref{prop:optimality_gap} sharpens it to $O(\|B\|_\infty^2/n)$ independent of $k$ via a refined ties argument. In the sparse case, the additive gap $\|\bproj{\pihat}\|_\infty^2$ could be further controlled, but we leave this direction to future work. We also discuss in \Cref{rk:xu_comparison} the relation with structured GW-barycenter methods \cite{xu_learning_2021b}: while these methods address graph averaging via entropically regularised barycenters, our goal is recovery of the latent graphon structure through near-deterministic low-dimensional couplings, making the two approaches related but not directly comparable. See \Cref{rk:xu_comparison} for details.

\subsection{Block-coordinate descent algorithm}\label{subsec:algo}

We now describe an efficient implementation  for solving \cref{eq:SRGWB}. Since the objective jointly optimizes over the block matrix and the assignment map, we adopt an alternating scheme that iteratively updates each component \cite{vincent-cuaz_semirelaxed_2022}. For a fixed labeling, the block matrix $B$ admits a closed-form solution \eqref{eq:block_matrix_and_weigths}, so no iterative optimization is required. The algorithm therefore alternates between a Gromov-Wasserstein transport update for $\pi$, and a closed-form barycenter update for $B$.
We give pseudo-code in \cref{alg:BCM_relaxed_SBM}. The inner step relies on the Frank-Wolfe algorithm, which converges to stationary points at rate $O(1/\sqrt{T})$ in the primal-dual gap for non-convex objectives \cite{lacoste-julien_convergence_2016} and is sketched in \cref{alg:SRGWB_pot}. The dominant cost is the conditional-gradient inner loop, which we parallelise on GPU through the \texttt{POT} library \cite{flamary_pot_2021a}; the closed-form $\tilde B$-update is $O(nk)$ and runs on either device.
\begin{algorithm}[h]
\caption{Block-Coordinate Minimization (BCM) for srGWB}
\label{alg:BCM_relaxed_SBM}
\begin{algorithmic}[1]

\STATE \textbf{Input:} Graph $\Edges \in \{0,1\}^{n\times n}$; number of blocks $k$
\STATE Initialize $\pi^{(0)} \in \Pi_{[n],k}$

\FOR{$t = 0,1,2,\dots, T-1$}
    \STATE
    \(
    B^{(t+1)}_{a,b}
    =
    \frac{
        \sum_{i,j=1}^{n} \Edges_{i,j}\, \pi^{(t)}_{i,a}\, \pi^{(t)}_{j,b}
    }{
        \sum_{i,j=1}^{n}\pi^{(t)}_{i,a}\, \pi^{(t)}_{j,b}
    },
    \) \hfill \textbf{(Update $B$)}
    \STATE
    \(
    \pi^{(t+1)}
    =
    \operatorname*{argmin}\limits_{\pi \in \Pi_{[n],k}}
    \sum\limits_{i,j=1}^{n} \sum\limits_{a,b=1}^{k}
    \bigl(\Edges_{i,j} - B^{(t+1)}_{a,b}\bigr)^2
    \pi_{i,a}\, \pi_{j,b},
    \)\hfill (\textbf{Optimize $\pi$:}  \Cref{alg:SRGWB_pot} )

\ENDFOR; \textbf{Output:} $\pi^T$
\end{algorithmic}
\end{algorithm}

\paragraph{Initialisation and choice of $k$.}\label{par:choice_k}
Since \eqref{eq:SRGWB} is non-convex, $\pi^{(0)}$ matters: \Cref{appendix:data_analysis} compares three strategies ($k$-means on the spectral embedding, spectral clustering, product/uniform), with $k$-means most stable on graphons and product fastest on SBMs (\Cref{fig:continous_all,fig:sbm_sparsity,fig:sbm_true_k}). We use $k$-means by default, since it is most stable across graphon settings.
The graphon rate uses oracle $k_{\rm opt}=\lceil n^{1/(\alpha\wedge1+1)}\rceil$ depending on the unknown smoothness $\alpha$. We use the following self-pruning scheme: pick a conservative upper bound $k_{\max}$, run \Cref{alg:BCM_relaxed_SBM} with $k=k_{\max}$, and report the effective block count $\widehat k=|\{a:\widehat\nu(a)>0\}|$ from the pruned right marginal. The free right marginal in \eqref{eq:SRGWB} acts as implicit complexity regularisation: blocks that do not improve the srGW objective receive zero mass and drop out, so $\widehat k\le k_{\max}$ adapts to the data without an explicit penalty. Across our \totalsynthexp synthetic settings, $k_{\max}=\lceil\sqrt n\rceil$ (graphons) and $k_{\max}=20$ (SBMs) are sufficient: $\widehat k$ tracks the effective block count, and the GW estimation error is essentially flat in $k_{\max}$ once $k_{\max}\ge 2 k_{\rm opt}$ (\Cref{fig:continous_all,fig:sbm_sparsity}). %

Across all \totalsynthexp synthetic settings, the BCM solver returns near deterministic couplings at convergence (the fraction of nodes carrying mass on more than one block is almost zero in every run; \Cref{appendix:data_analysis}). \Cref{lem:almostmap,prop:optimality_gap} guarantee that any local minimum lies within $O(1/n)$ of a deterministic assignment in cost. Still, these results do not fully explain why gradient-based methods typically converge to deterministic solutions -- which can arguably be seen as a form of implicit regularization, with gradient descent converging to more structured solutions.

\section{Data analysis}\label{sec:data_analysis}

Synthetic experiments use binary graphs; real-data examples use weighted networks fitted via the expected adjacency matrix \cite{donier-meroz_graphon_2023a} (the method extends naturally to decorated graphons \cite{dufour_inference_2024a}).

\paragraph{Synthetic experiments.}
We consider \numgraphons{} graphons and \numsbms{} SBMs, running $20$ Monte Carlo replicates per (model, initialisation, $n$). \Cref{fig:synth_main} shows the decaying error rate for $n\in\{100,\dots,10'000\}$ on a periodic graphon -- well beyond the $n\!\lesssim\!10^3$ scale at which graphon-estimation methods are typically benchmarked \cite{chan_consistent_2014a,zhang_estimating_2017a,chatterjee_matrix_2015a,xu_learning_2021b} -- matching the theoretical rate of \Cref{subsec:learning}; analyses for the remaining models are in \Cref{subsec:sims}. %
\begin{figure}[h!]
    \centering
    \begin{minipage}{0.49\linewidth}
        \centering
        \includegraphics[width=\linewidth]{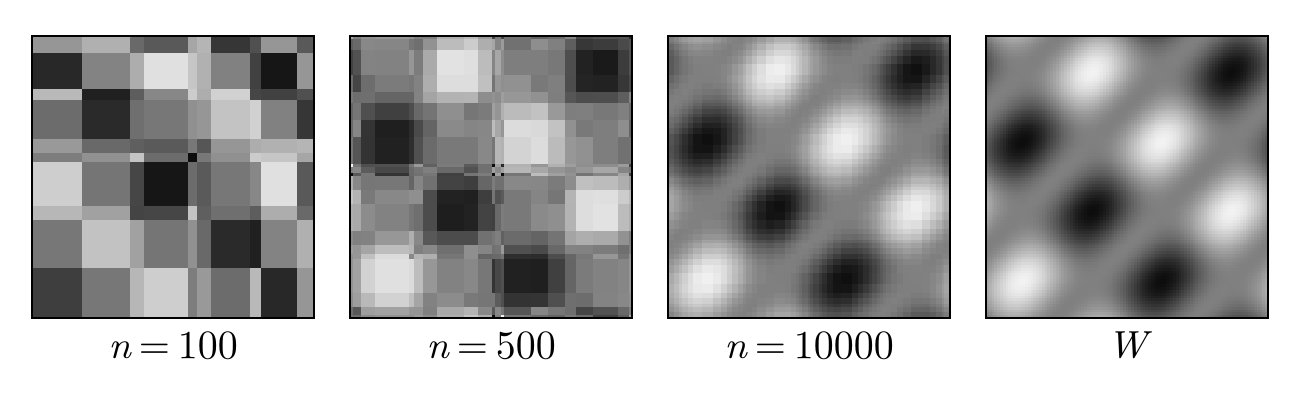}
    \end{minipage}
    \hfill
    \begin{minipage}{0.49\linewidth}
        \centering
        \includegraphics[width=\linewidth]{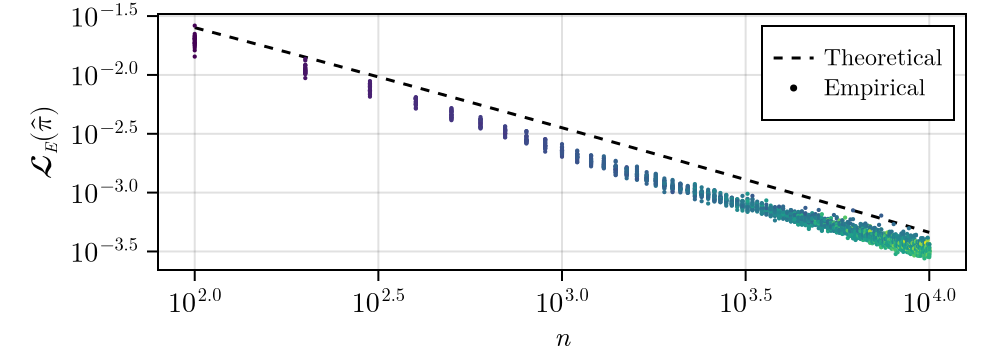}
    \end{minipage}
    \caption{Left: estimator for increasing network size in the 'G4' example. Right: empirical loss, matching theoretical rates. Obtained from $20$ Monte Carlo runs with clustering-based initialization. Similar and additional plots for the other considered models are found in \Cref{subsec:sims}.}
    \label{fig:synth_main}
    \vspace{-2em}
\end{figure}
\begin{table}[h]
\caption{GW distance to truth ($\times 10^{-3}$) at $n = 2000$, mean$_{\pm\text{std}}$ over 20 runs. \textbf{Bold}: best per setting; \emph{italic}: second best. \textbf{Ours} uses the best-performing initialization per setting, and perform uniformly well across all settings. Additional details can be found in \Cref{appendix:data_analysis}, including timing results.\newline}
\label{tab:benchmark}
\centering
\renewcommand{\arraystretch}{1.25}
\setlength{\tabcolsep}{3pt}
\resizebox{\textwidth}{!}{%
\begin{tabular}{lcccc|cccccc}
\toprule
 & \multicolumn{4}{c|}{SBMs} & \multicolumn{6}{c}{Graphons} \\
\cmidrule(lr){2-5}\cmidrule(l){6-11}
Method & S1 & S2 & S3 & S4 & G1 & G2 & G3 & G4 & G5 & G6 \\
\midrule
LG            & $88.4_{\pm 2.9}$              & $88.6_{\pm 2.6}$              & $71.9_{\pm 17}$               & $40.7_{\pm 14}$               & $2.8_{\pm 0.66}$              & $2.2_{\pm 0.39}$              & $1.3_{\pm 0.12}$              & $42.6_{\pm 1.1}$              & $42.8_{\pm 1.2}$              & $55.7_{\pm 9.8}$ \\
NBDsmooth     & $0.82_{\pm 0.007}$            & $0.82_{\pm 0.007}$            & $0.82_{\pm 0.004}$            & $0.81_{\pm 0.010}$            & $0.70_{\pm 0.012}$            & $1.0_{\pm 0.009}$             & $1.2_{\pm 0.011}$             & $\emph{0.88}_{\pm 0.007}$ & $\mathbf{0.91}_{\pm 0.010}$  & $0.93_{\pm 0.015}$ \\
SAS           & $81.3_{\pm 11}$               & $82.2_{\pm 9.8}$              & $0.97_{\pm 0.19}$             & $5.8_{\pm 3.5}$               & $\emph{0.46}_{\pm 0.011}$ & $\mathbf{0.62}_{\pm 0.009}$  & $1.0_{\pm 0.019}$             & $42.0_{\pm 1.1}$              & $41.9_{\pm 1.1}$              & $54.3_{\pm 11}$ \\
USVT          & $\emph{0.32}_{\pm 0.007}$ & $\emph{0.32}_{\pm 0.007}$ & $\emph{0.32}_{\pm 0.004}$ & $\emph{0.45}_{\pm 0.010}$ & $\mathbf{0.18}_{\pm 0.005}$   & $1.7_{\pm 0.07}$              & $\mathbf{0.73}_{\pm 0.016}$  & $1.3_{\pm 0.018}$             & $2.7_{\pm 0.59}$              & $\mathbf{0.34}_{\pm 0.008}$ \\
\textbf{Ours} & $\mathbf{0.013}_{\pm 0.0003}$ & $\mathbf{0.31}_{\pm 0.0003}$  & $\mathbf{0.31}_{\pm 0.0004}$  & $\mathbf{0.24}_{\pm 0.001}$   & $0.61_{\pm 0.05}$              & $\emph{0.80}_{\pm 0.06}$ & $\emph{0.81}_{\pm 0.10}$ & $\mathbf{0.72}_{\pm 0.12}$    & $\emph{1.2}_{\pm 0.03}$  & $\emph{0.92}_{\pm 0.05}$ \\
\bottomrule
\end{tabular}%
}
\end{table}

\paragraph{Baseline comparison.}
\Cref{tab:benchmark} benchmarks our estimator at $n=2000$ against the standard SBM/graphon baselines with mature implementations that target the same recovery loss and scale to $n=2000$: Largest-Gap (LG, \cite{channarond_classification_2012a}), Neighbourhood Smoothing (NBDsmooth, \cite{zhang_estimating_2017a}), Sort-and-Smooth (SAS, \cite{chan_consistent_2014a}), and USVT \cite{chatterjee_matrix_2015a}. Our method delivers the strongest overall performance: it attains the smallest GW distance on the SBM settings (often by one to two orders of magnitude over every competitor), while ranking second on four of the five remaining smooth graphons (within a small constant factor of the leader) and being computationally efficient (see \cref{fig:timing}). Crucially, ours is the only estimator that performs uniformly well across both graphons \emph{and} block models, while being interpretable (in opposition to USVT): each baseline collapses on at least one regime. The histogram estimator \cite{olhede_network_2014b}-- which subsumes spectral-clustering with local refinement -- and the variational MLE \texttt{sparseBM} \cite{frisch_sparsebm_2022}  did not terminate within our computational budget; restricted comparisons for $n\leq 300$ are in \Cref{fig:benchmark}, and our method is order of magnitude faster. The real-data examples lack ground truth, so we evaluate qualitatively against domain priors (\Cref{appendix:data_analysis}).

\paragraph{EEG dataset.}
We apply our method to recordings from $64$ electrodes on $20$ subjects ($10$ alcoholic, $10$ control) sampled at $256$ Hz over $1$ second under three stimulus paradigms \cite{begleiter_eeg_1995}. Edge weights are Centered Kernel Alignment (CKA) between replicated electrode time series, with a linear kernel \cite{kornblith_similarity_2019}; we apply srGWB with $k=40$ blocks. \Cref{fig:eeg_heatmaps} displays the fitted co-activation structures and their mapping to brain regions; \Cref{fig:eeg_brains} shows the alcoholic-vs-control difference, with increased co-activation in alcoholic subjects between frontal and parietal/occipital regions, a pattern commonly interpreted as a compensatory mechanism \cite{chanraud2013remapping, cao2014eeg}. See \Cref{appendix:data_analysis} for further analysis.
\begin{figure}[h]
    \centering
    \begin{subfigure}{0.43\linewidth}
        \centering
        \includegraphics[width=\linewidth]{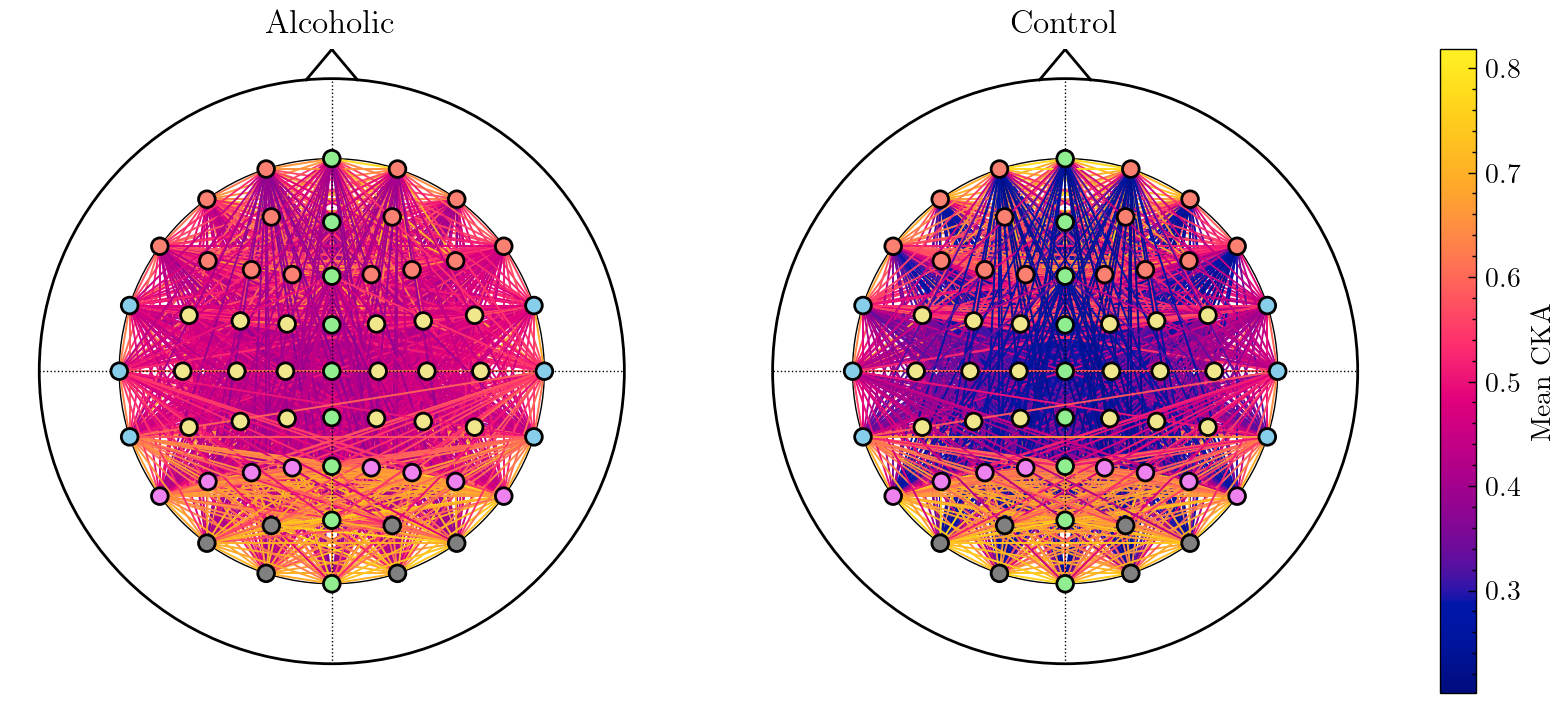}
        \caption{Inferred latent connectivity in alcoholic (left) and control (right).}
        \label{fig:eeg_brains}
    \end{subfigure}
    \hfill
    \begin{subfigure}{0.43\linewidth}
        \centering
         \includegraphics[width=\linewidth]{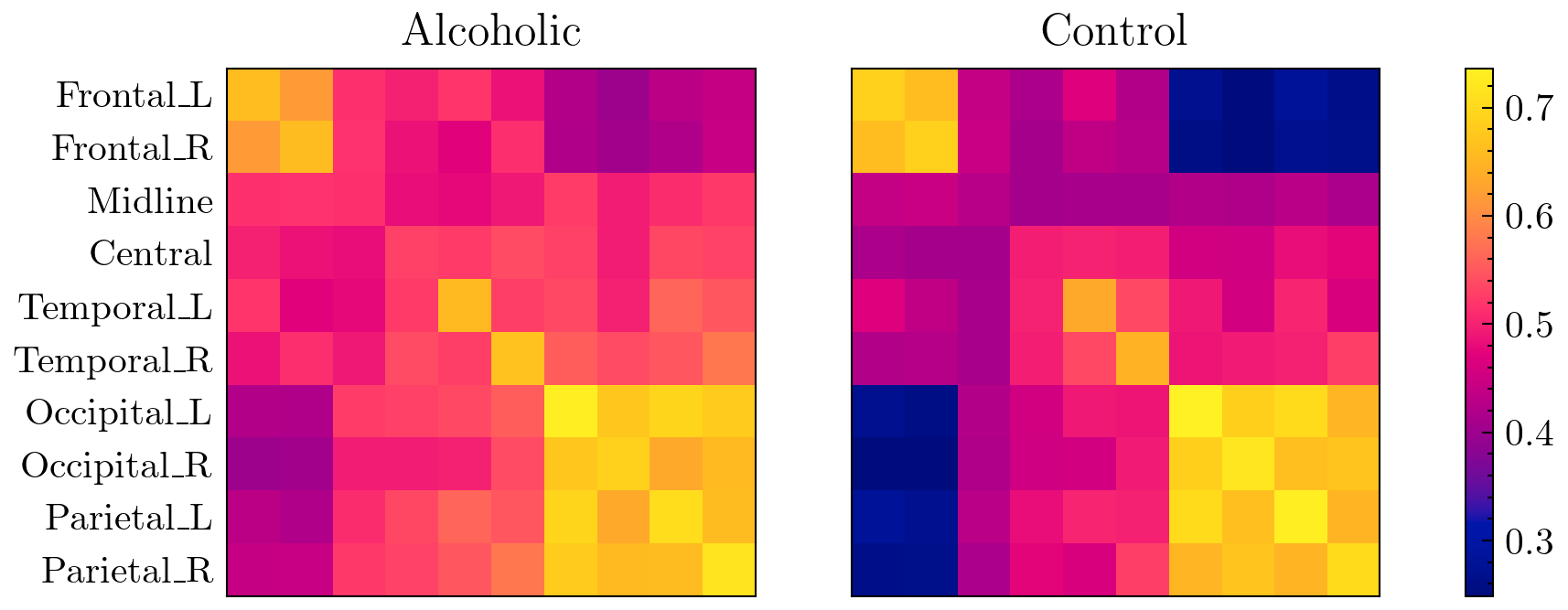}
        \caption{Mapped co-activation structures in alcoholic (left)  and control (right).  }
        \label{fig:eeg_heatmaps}
    \end{subfigure}
    \caption{EEG dataset: inferred block connectivity (left) and co-activation structures (right) for the first stimulus condition.}
    \vspace{-2em}
\end{figure}

\paragraph{OpenFlight Airport Data.}
We analyse the OpenFlight airport network \cite{openflights_airports}, with edges encoding the existence of direct flights, partitioning $3257$ airports into $23$ clusters. \Cref{fig:airports_map1} shows the inferred block-level connectivity; the full geographically-laid-out clustering is in \Cref{appendix:data_analysis} and reveals strong spatial coherence and major hubs.
\begin{figure}[h!]
    \centering
    \includegraphics[width=0.43\linewidth]{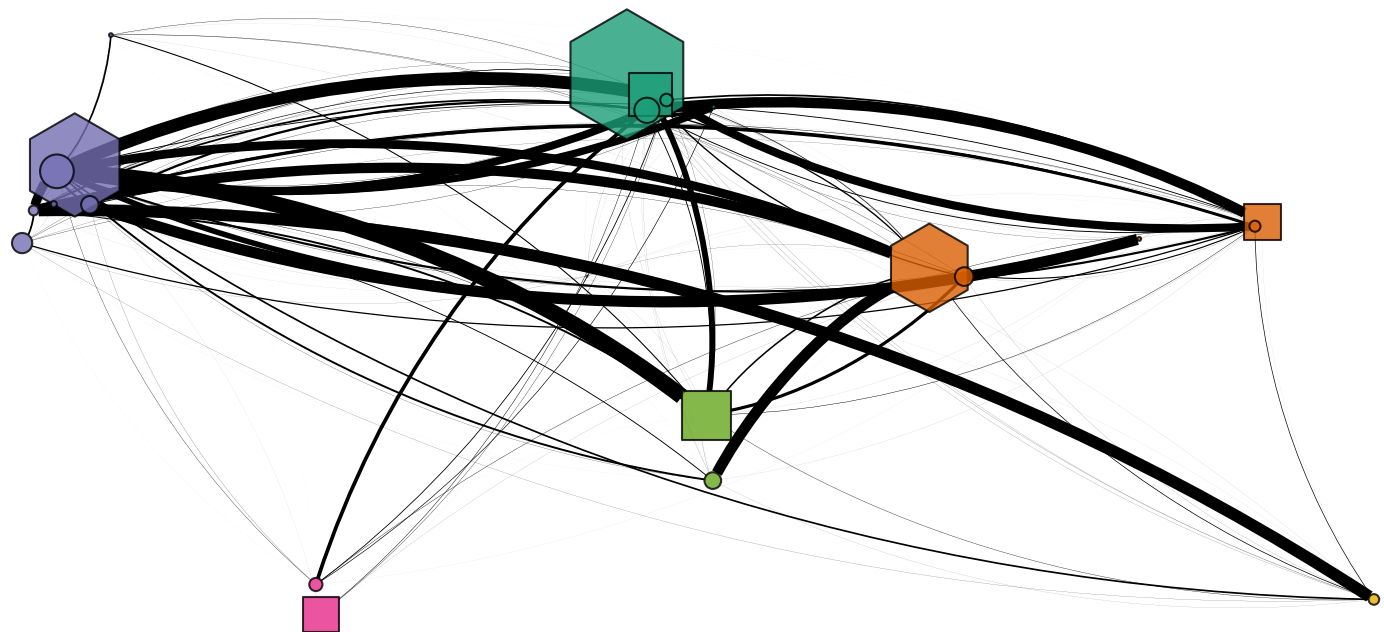}
    \includegraphics[width=0.43\linewidth]{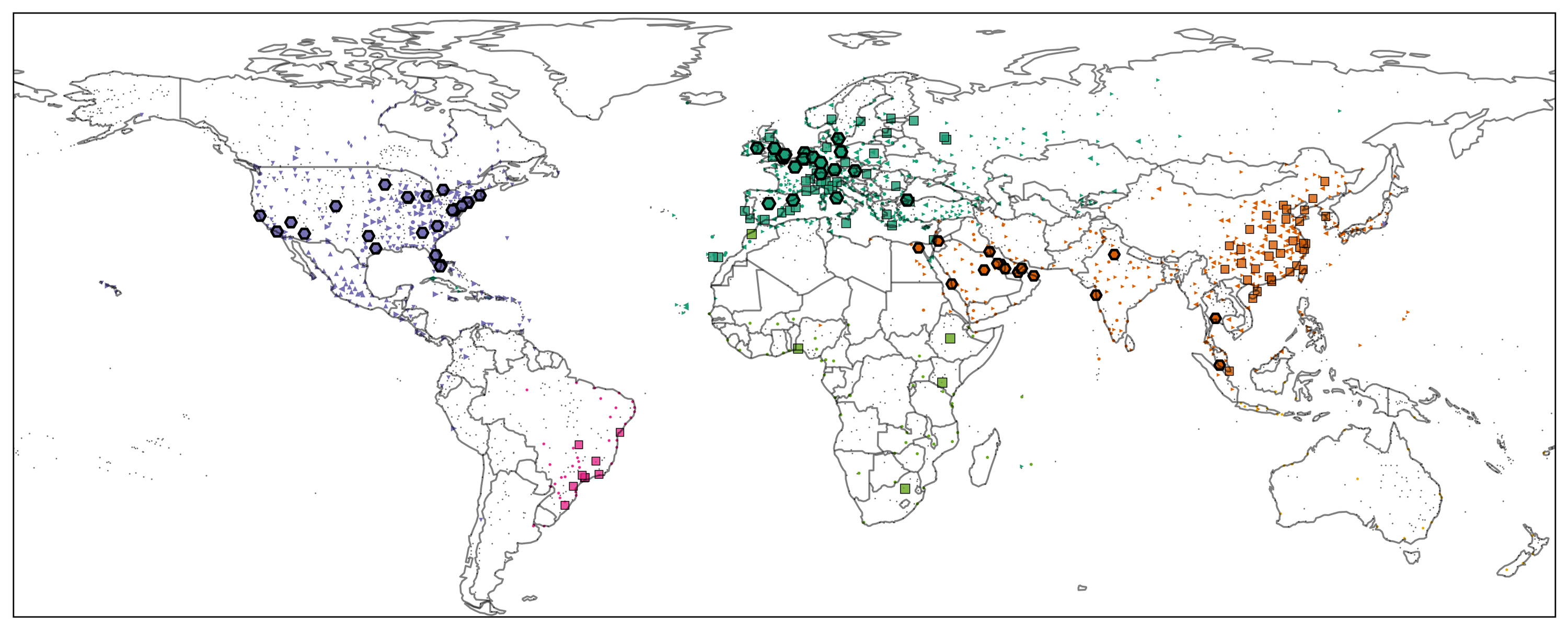}
    \caption{OpenFlight network: inferred block-level connectivity. Geographic map in \Cref{appendix:data_analysis}.}
    \label{fig:airports_map1}
\end{figure}

\section{Discussion}\label{sec:discussion}
We introduced a tractable framework for latent network inference via semi-relaxed Gromov–Wasserstein barycenters. Relaxing hard assignments to probabilistic couplings turns SBM/graphon estimation into a continuous optimization that avoids NP-hard combinatorial search, with a regime-agnostic optimality gap of order $\|B\|_\infty^2/n$. This yields minimax-optimal rates for SBMs and Hölder graphons in both dense \citep{gao_rateoptimal_2015a} and sparse \citep{klopp_oracle_2017b} regimes, via a scalable conditional gradient algorithm. Scope and natural extensions are discussed in \Cref{par:choice_k}. An important and promising extension would be to incorporate in the method and its analysis a penalty for the number of blocks selected by the algorithm. Such an extension would offer a novel framework for adaptively selecting the preferred complexity by which to summarize network data. However, the addition of a sparsity-promoting regularization to the semi-relaxed Gromov-Wasserstein problem \eqref{eq:SRGWB} significantly modifies the optimization landscape and requires a separate, careful analysis. Finally, while the numerical examples are mainly illustrative, we see the proposed framework as an important step toward more general nonparametric methods for network data, for example, for sieve bootstrap or data augmentation in networks with edge and node attributes.

\bibliographystyle{unsrtnat}
\bibliography{ref.bib,references_manual.bib}

\appendix

\section{Background}\label{appendix:background}

\subsection{Generative models for networks}\label{sec:genmodelsforgraphs}

Generative models for networks provide a probabilistic framework for describing graph-valued data. In this section we briefly review classical models relevant to our setting: namely, stochastic block models, graphons, and their extensions. We refer the readers to \citet{orbanz_bayesian_2015b} for details.

\paragraph{Stochastic block models (SBMs).}
The stochastic block model assumes that each node belongs to one of a finite number of latent groups, and that the probability of an edge between two nodes depends only on their group memberships \cite{holland_stochastic_1983a}. Formally, let $K \in \mathbb{N}$ denote the number of blocks, let $z_i \in \{1,\dots,K\}$ be the latent block assignment of node $i$, and let $B \in [0,1]^{K \times K}$ be a symmetric matrix of connection probabilities. Conditional on the assignments $(z_i)_{i=1}^{n}$, edges are generated independently according to
\[
\mathbb{P}(\Edges_{ij} = 1 \mid z_i, z_j) = B_{z_i,z_j}, \qquad 1 \le i < j \le n.
\]
SBMs capture community structures and provide a tractable model for clustering and inference. They admit many extensions, including degree-corrected SBMs, mixed-membership SBMs, and dynamic or hierarchical variants (see e.g., \cite{zhao_consistency_2012a}). 
Importantly, observe that SBMs correspond to graphons that are piecewise constant with respect to a finite partition of $[0,1]$. As such, they can be viewed as finite-dimensional approximations to more general graphon models.

\paragraph{Graphons.}
Graphons provide a nonparametric framework for modeling large graphs and serve as limit objects for sequences of graphs with growing size \cite{lovasz_large_2012a}. Formally, a graphon is defined as a measurable function $W : [0,1]^2 \to [0,1]$. Given $W$, a random graph with $n$ nodes can be generated by sampling latent variables $u_1,\dots,u_n \stackrel{\text{i.i.d.}}{\sim} \mathrm{Unif}[0,1]$ and then, conditionally on these variables, drawing edges independently as
\[
\mathbb{P}(\Edges_{ij} = 1 \mid u_i, u_j) = W(u_i,u_j), \qquad 1 \le i < j \le n.
\]
Graphons are identifiable only up to measure-preserving transformations of the unit interval. That is, if $\phi : [0,1] \to [0,1]$ is measure-preserving, then $W$ and $W^\phi(u,v) = W(\phi(u),\phi(v))$ define the same distribution over graphs \cite{janson_graphons_2011a}. This inherent non-identifiability motivates the use of comparison tools, such as the cut distance or Gromov--Wasserstein distance, that are invariant under such transformations.
A useful way to interpret graphons is as infinite-dimensional analogues of SBMs, where the discrete block assignments are replaced by continuous latent positions. Many statistical properties of graphons depend on regularity assumptions.

A graphon $W$ is said to be $(\alpha, L)$-Hölder continuous for some $\alpha \in (0,1]$ and $L > 0$ if
\[
|W(x,y) - W(x',y')| \le L\bigl(|x-x'| + |y-y'|\bigr)^\alpha \quad \text{for all } (x,y),(x',y') \in [0,1]^2.
\]
Continuity and Hölder smoothness captures the idea that nodes with close latent positions should have somewhat close connectivity patterns. This assumption enables approximation of smooth graphons by block models with a controlled number of blocks and underlies convergence rates for graphon estimation \cite{klopp_oracle_2017b}. In our context, smoothness assumptions help relate finite-sample graph observations to their underlying graphon representations.

\paragraph{Decorated graphs.}
Many real-world networks carry additional information beyond binary connectivity, such as node features, and edge attributes or weights. Decorated and probability graphons extend classical graphons by incorporating this supplementary information \cite{lovasz_limits_2010a,abraham_probabilitygraphons_2023a}. 
Formally, we consider a node feature space $\mathcal{X}$ and an edge space $\Espace$ and, for each node $i$ and $j$, sample $x_i \in \mathcal{X}$ and $\Edges_{ij} \in \Espace$. The distributions of the node features and edges are still independent conditionally on the latent variables;
for example, a decorated graphon model may specify a joint distribution
\[
(u_i, x_i) \sim \mu, \qquad \Edges_{ij} \mid u_i,u_j \sim W(u_i,u_j),
\]
where the node attributes $(x_i)$ are informative about the latent positions $(u_i)$ and $W(u_i,u_j)$ is a probability measure supported on $\Espace$.  
It is possible to define a distance for general networks encompassing all features previously described \cite{bauer_zgromovwasserstein_2025a}, and our results extend naturally to the case where we do not consider node features and $\Espace$ is compact and bounded as described in \cite{dufour_inference_2024a}.

\subsection{Gromov--Wasserstein distance}\label{appendix:gw}

\paragraph{Gromov--Wasserstein distance.}
Define a metric measure space as a triple $(X, d_X, \mu_X)$, where $(X,d_X)$ is a metric space and $\mu_X$ is a probability measure on $X$. Given two metric measure spaces $(X,d_X,\mu_X)$ and $(Y,d_Y,\mu_Y)$, the squared $2$-Gromov--Wasserstein distance is defined as
\[
\mathrm{GW}_2^2(\mu_X, \mu_Y)
=
\inf_{\pi \in \Pi(\mu_X,\mu_Y)}
\iint_{X^2 \times Y^2}
\bigl| d_X(x,x') - d_Y(y,y') \bigr|^2
\, \mathrm{d}\pi(x,y)\,
\mathrm{d}\pi(x',y'),
\]
where $\Pi(\mu_X,\mu_Y)$ denotes the set of couplings (joint probability measures) with marginals $\mu_X$ and $\mu_Y$.

Unlike classical optimal transport, the GW distance does not require the spaces $X$ and $Y$ to be embedded in a common ambient space. Instead, it aligns points so as to best preserve pairwise distances, making it invariant to isometries and, more generally, to measure-preserving re-labelings. We refer the reader to \citet{memoli_gromov_2011a} for details.

\paragraph{GW distance between graphs.}
In the context of graphs, nodes are typically equipped with the uniform probability measure, and the metric is induced by the adjacency matrix or edge weights. Let
\(
\Edges \in \Espace^{n \times n}
\)
and
\(
F \in \Espace^{m \times m}
\)
denote the adjacency matrices of two (possibly weighted) graphs. The GW distance between them can be written as
\[
\mathrm{GW}_2^2(\Edges, F)
=
\min_{\pi \in \Pi([n],[m])}
\sum_{i,i'=1}^n \sum_{j,j'=1}^m
\bigl| \Edges_{i,i'} - F_{j,j'} \bigr|^2
\, \pi_{i,j}\, \pi_{i',j'},
\]
where $\Pi([n],[m])$ denotes the set of couplings whose left and right marginals are the discrete uniform distributions on $\{1,\dots,n\}$ and $\{1,\dots,m\}$, respectively.
This formulation makes explicit that GW compares graphs up to node permutations, and relates GW to the Mean Squared Error (MSE) used in \citet{gao_rateoptimal_2015a}.
See \citet{peyre_gromovwasserstein_2016a, xu_gromovwasserstein_2019, vayer_fused_2020,bauer_zgromovwasserstein_2025a} for details.

\paragraph{Semi-relaxed Gromov--Wasserstein distance.}
In some cases, it may be desirable to relax one of the marginal constraints. This leads to the \emph{semi-relaxed Gromov--Wasserstein} distance \cite{vincent-cuaz_semirelaxed_2022}.
Given again metric measure spaces $(X,d_X,\mu_X)$ and $(Y,d_Y,\mu_Y)$, the semi-relaxed GW problem is given by  relaxing the marginal constraint on $Y$:
\[
\mathrm{srGW}_2^2(\mu_X, \mu_Y)
=
\inf_{\substack{\pi \ge 0 \\ (\pi)_X = \mu_X}}
\iint_{X^2 \times Y^2}
\bigl| d_X(x,x') - d_Y(y,y') \bigr|^2
\, \mathrm{d}\pi(x,y)\,
\mathrm{d}\pi(x',y'),
\]
where $(\pi)_X$ denotes the marginal of $\pi$ on $X$. 
In the discrete graph setting, this translates as:
\(
\Edges \in \Espace^{n \times n}
\)
and
\(
F \in \Espace^{m \times m},
\)
the semi-relaxed GW distance takes the form
\[
\mathrm{srGW}_2^2(\Edges, F)
=
\min_{\substack{\pi \in \mathbb{R}_+^{n \times m} \\
\sum_{j=1}^m \pi_{i,j} = \tfrac{1}{n}}}
\sum_{i,i'=1}^n \sum_{j,j'=1}^m
\bigl| \Edges_{i,i'} - F_{j,j'} \bigr|^2
\, \pi_{i,j}\, \pi_{i',j'}.
\]
The semi-relaxed formulation retains invariance to relabeling of the source graph while allowing greater flexibility in matching to the target. This relaxation is particularly useful when comparing finite graphs to graphons, or when the target space represents a reference structure rather than a fixed empirical distribution.

\paragraph{Semi-relaxed Gromov--Wasserstein barycenters.}
Let a collection of metric measure spaces $(X_\ell,d_\ell,\mu_\ell)_{\ell=1}^L$ and weights $(\lambda_\ell)_{\ell=1}^L$, $L\geq 1$, such that $\lambda_{\ell} \geq 0$ and $\sum_{\ell=1}^{L}\lambda_{\ell} = 1$. A srGW barycenter is defined as a space $(Y,d_Y,\mu_Y)$ minimizing the weighted sum of semi-relaxed GW discrepancies
\[
\sum_{\ell=1}^L \lambda_\ell \, \mathrm{srGW}_2^2(\mu_\ell,\mu_Y).
\]
In contrast to classical GW barycenters, the semi-relaxed formulation fixes the marginal on each space while allowing flexibility in the barycenter measure, which is particularly advantageous when aggregating graphs of varying sizes or when estimating a population-level structure \cite{vincent-cuaz_semirelaxed_2022}. 
In the graph or graphon setting, srGW barycenters can be interpreted as representative structures that best preserve the relational patterns of the input graphs up to measure-preserving transformations. This makes srGW barycenters well suited for tasks such as averaging networks, template learning, and estimating graphons.

\section{Proofs}\label{appendix:proofs}
\begin{proof}[Proof of Lemma~\ref{lem:almostmap}]
    We adapt the ideas introduced in \cite{murray_probabilistic_2025} to our context.
    Let $\tilde\pi$ be a local-minima of $\cL_{\Edges,B}$ 
    Denote 
    \begin{align*}
    \cJ_{\tilde \pi}: (i,a) \in [n] \times [k] \mapsto \sum_{j=1}^n \sum_{b=1}^k \|\Edges_{i,j}- B_{a,b}\|^2_2 \, \tilde\pi_{j,b}.
    \end{align*}
    Fix $i \in [n]$ and $a \in [k]$ and assume
    that
    $\tilde\pi_{i,a} > 0$.
    By contradiction, suppose $a \notin \argmin_{\circ \in [k]} \cJ_{\tilde \pi}(i, \circ)$.
    Let $a' \in \argmin_{\circ \in [k]} \cJ_{\tilde \pi}(i, \circ)$ and define the perturbation 
    $\psi = \frac{\epsilon}{n}\tilde \pi_{i,a}\left(\delta_{i,a'}-\delta_{i,a}\right)$ for some $\epsilon > 0$. Then, observe that:
    \begin{align*}
        \cL_{\Edges, B}(\tilde \pi + \psi) - \cL_{\Edges, B}(\tilde \pi)
        &=  
        2 {\sum_{j,j'=1}^n \sum_{d,d'=1}^k
        \|\Edges_{j,j'} - B_{d,d'}\|_2^2 \, \pi_{j,d}\psi_{j',d'}}
        \\ &\qquad+ {\sum_{j,j'=1}^n \sum_{d,d'=1}^k
        \|\Edges_{j,j'} - B_{d,d'}\|_2^2 \, \psi_{j,d}\psi_{j',d'}}.
    \end{align*}
    By definition of $\psi$, the first term reduces to $\frac{\epsilon}{n}(\cJ_{\tilde \pi} (i, a') - \cJ_{\tilde \pi} (i, a)) < 0$. The second term is quadratic in $\epsilon/n$: we can find $\epsilon>0$ small enough such that $\cL_{\Edges, B}(\tilde \pi + \psi) - \cL_{\Edges, B}(\tilde \pi) < 0$, which contradicts the minimality of $\tilde \pi$. This concludes the proof.
\end{proof}

\begin{proof}[Proof of \Cref{prop:optimality_gap}]
Suppose a (local) optimal embedding $\pi$ maps a single node $i$ in $E$ to multiple blocks in $B$. Hence, there exist $2 \leq k_i \leq k$ blocks $\omega(i):=\{a_{k_1},\ldots,a_{k_i}\}\in [k]$ such that $\pi_{i,a}>0$ for all $a \in \omega(i)$.
By Lemma~\ref{lem:almostmap}, necessarily it must be that for any $a \in \omega(i)$
\begin{equation}\label{eq:fooc}
a \in \argmin_{d \in [k]} \sum_{j=1}^n \sum_{b'=1}^k 
\|\Edges_{i,j}- B_{d,b'}\|^2_2
\, \pi_{j,b'}.
\end{equation}
Without loss of generality, assume that $\omega(i) = \{1,\ldots,k_i\}$ and that $\pi_{i,1} = \max_{a \in \omega(i)} \pi_{i,a}$.
We can construct a new coupling $\pi'$ from $\pi$ by transporting all the mass in $\left\{(i,b)\right\}_{b=2}^{k_i}$ to $(i,1)$, without violating our first-order condition:
\[
\pi' = \pi + \underbrace{\sum_{b=2}^{k_i}\pi_{i,b}\left(\delta_{i,1}-\delta_{i,b}\right)}_{\displaystyle:=\psi},
\]
with $\psi$ being a signed measure with zero total mass. We emphasize that $\pi'$ is still a valid probabilistic mapping:
\begin{enumerate}
    \item \textbf{(Nonnegativity):} For $(j,d)\notin\{(i,a),(i,b)\}$ we have $\pi'_{j,d}=\pi_{j,d}\ge0$. For $(j,d)=(i,a)$,
\(\pi'_{i,a}=\pi_{i,a}+ \sum_{b=2}^{k_i}\pi_{i,b}\ge0\). For $(j,d)=(i,b)$,
\(\pi'_{i,b}=\pi_{i,b}-\pi_{i,b}=0\). Hence,
\(\pi'_{j,d}\ge0\) for all \(j,d\).
\item \textbf{(Left marginal)}
By assumption \(\pi\) has left marginal uniform, i.e.
\(\sum_{d=1}^k\pi_{j,d}=1/n\) for every \(j\). For any fixed \(j\neq i\) we have \(\psi_{j,d}=0\)
for all \(d\), hence
\(\sum_{d}\pi'_{j,d}=\sum_{d}\pi_{j,d}=1/n\).
For \(j=i\),
\(
\sum_{d=1}^k \pi'_{i,d}
= \sum_{d}\pi_{i,d} + \sum_{d}\psi_{i,d}
= \frac{1}{n} + 0 = \frac{1}{n},
\) since \(\psi\) has zero total mass.
Therefore, the left marginal of \(\pi'\) is still uniform on \([n]\).
\item \textbf{(Total mass)} Since the left marginal is preserved for every \(j\), the total mass \(\sum_{j,d}\pi'_{j,d}\) remains
\(\sum_j (1/n)=1\), so \(\pi'\) is normalized as a coupling.
\end{enumerate}

Since $\cL_{\Edges,\tilde B}$ is quadratic in the coupling, we can split it into what is essentially a Taylor decomposition in the first and second-order variations of $\cL_{\Edges,\tilde B}$. Explicitly, one can see that:
\begin{align*}
    \cL_{\Edges,\tilde B}(\pi')
    &= \cL_{\Edges,\tilde B}(\pi + \psi ) \\
    &= \sum_{j,j'=1}^n \sum_{d,d'=1}^k
       \|\Edges_{j,j'} - B_{d,d'}\|_2^2 \, (\pi + \psi)_{j,d}\, (\pi + \psi)_{j',d'} \\
    &= \sum_{j,j'=1}^n \sum_{d,d'=1}^k
       \|\Edges_{j,j'} - B_{d,d'}\|_2^2 \, \pi_{j,d}\pi_{j',d'} 
        \\
    &\qquad\qquad + 2 \sum_{j,j'=1}^n \sum_{d,d'=1}^k
       \|\Edges_{j,j'} - B_{d,d'}\|_2^2 \, \pi_{j,d}\psi_{j',d'}
       + \sum_{j,j'=1}^n \sum_{d,d'=1}^k
       \|\Edges_{j,j'} - B_{d,d'}\|_2^2 \, \psi_{j,d}\psi_{j',d'} \\
    &= \cL_{\Edges,\tilde B}(\pi) 
       + 2 \underbrace{\sum_{j,j'=1}^n \sum_{d,d'=1}^k
       \|\Edges_{j,j'} - B_{d,d'}\|_2^2 \, \pi_{j,d}\psi_{j',d'}}_{\mathrm{(D1)}} 
       + \underbrace{\sum_{j,j'=1}^n \sum_{d,d'=1}^k
       \|\Edges_{j,j'} - B_{d,d'}\|_2^2 \, \psi_{j,d}\psi_{j',d'}}_{\mathrm{(D2)}}.
\end{align*}

We first show that $\mathrm{(D1)} = 0$. Since $\psi = \sum_{b=2}^{k_i}\pi_{i,b}\left(\delta_{i,1}-\delta_{i,b}\right)$ the sum collapses to:
\begin{align*}
    \mathrm{(D1)}
    &= \sum_{j,j'=1}^n \sum_{d,d'=1}^k
       \|\Edges_{j,j'} - B_{d,d'}\|_2^2 \, \pi_{j,d}\psi_{j',d'} 
    \\ & = \sum_{b=2}^{k_i}\pi_{i,b}\left(\sum_{j=1}^n \sum_{d=1}^k
       \|\Edges_{j,i} - B_{d,1}\|_2^2 \, \pi_{j,d}-\sum_{j=1}^n \sum_{d=1}^k
       \|\Edges_{j,i} - B_{d,b}\|_2^2 \, \pi_{j,d}\right)  
\end{align*}
and the last displayed equation is null by \eqref{eq:fooc}.

\smallskip

Next, we upper bound $\mathrm{(D2)}$.
\begin{align*}
    \mathrm{(D2)}
    &= \sum_{j,j'=1}^n \sum_{d,d'=1}^k
       \|\Edges_{j,j'} - B_{d,d'}\|_2^2 \, \psi_{j,d}\psi_{j',d'} 
    \\ & = 
    \sum_{b=2}^{k_i} \pi_{i,b}^2 \left(\|\Edges_{i,i} - B_{b,b}\|_2^2+\|\Edges_{i,i} - B_{1,1}\|_2^2\right)
    - 2 \sum_{2 \leq b < b' \leq k_i} \pi_{i,b}\pi_{i,b'} \|\Edges_{i,i} - B_{b,b'}\|_2^2 
    \\ & = 
    \sum_{b=2}^{k_i} \pi_{i,b}^2 \left(\| B_{b,b}\|_2^2+\|B_{1,1}\|_2^2\right)
    - 2 \sum_{2 \leq b < b' \leq k_i} \pi_{i,b}\pi_{i,b'} \|B_{b,b'}\|_2^2 \\
    & \leq
   \min_{b=2,\ldots,k_i} \left\{\pi_{i,b}^2\right\} \sum_{b=1}^{k_i} \left(\|B_{b,b}\|_2^2 + \|B_{1,1}\|_2^2 
     \right)
     \\ & \leq \min_{b=2,\ldots,k_i} \left\{\pi_{i,b}\right\}^2k_i 2\|B\|_\infty^2 \leq 2\left(\frac{1}{nk_i}\right)^2 k_i\|B\|_\infty^2 = \frac{2\|B\|_\infty^2}{n^2 k_i} \leq \frac{\|B\|_\infty^2}{n^2},
\end{align*}
where we used that there are no self-loops in $E$, $k_i\geq2$, and that $\sum_{d=1}^{k_i} \pi_{i,d} = \frac{1}{n}$.

\smallskip

Repeating this operation for all \(i \in [n]\) manifests a deterministic mapping with increased cost bounded by \(\|B\|_\infty^2/n\), hence proving the result.

\end{proof}

\begin{lemma}
    \label{lem:upper_bound_mappings}
    Let $\pihat$ be the optimal probabilistic estimator and $\operatorname{z}(\pihat)$ be the rounded deterministic labeling estimator, according to \cref{eq:rounding}.
    Consider the solution over deterministic couplings $\zols \in \argmin_{z \in \mathcal{Z}_{n,k}}\cL\left(\Edges, \bproj{z} \right)$.
    Then, we have:
    \begin{equation*}
        \cL_{\Edges}\left(\zols\right) \leq \cL_{\Edges}\left(\operatorname{z}(\pihat)\right) \leq \cL_{\Edges}\left(\zols\right) + O(1/n)
    \end{equation*}
\end{lemma}

\begin{proof}[Proof of \Cref{lem:upper_bound_mappings}]
    From the definition of $\zols$, and since $\operatorname{z}(\pihat)$ is a feasible point for the optimization problem defining $\zols$, we have
    $$
        \cL_{\Edges}\left(\zols\right) \leq \cL_{\Edges}\left(\operatorname{z}(\pihat)\right).
    $$
    Now by definition of $\hat{\pi}$, we have 
    $$        \cL_{\Edges}\left(\hat{\pi} \right) \leq \cL_{\Edges}\left(\zols \right).$$
    Finally, using the definition of $\operatorname{z}(\pihat)$ and the optimality gap in \Cref{prop:optimality_gap}, we have
    $$
        \cL_{\Edges}\left(\operatorname{z}(\pihat)\right) \leq \cL_{\Edges,\bproj{\hat\pi}}\left(\operatorname{z}(\pihat)\right) \leq \cL_{\Edges}\left(\hat{\pi}\right) + D_n \leq \cL_{\Edges}\left(\zols\right) + D_n.
    $$ 
    where \(D_n= \|\bproj{\hat{\pi}}\|_\infty^2/n\).
\end{proof}

\begin{proof}[Proof of \Cref{thm:estimation_1}, SBM]
    We define
    \begin{equation*}
        \thetahat_{ij} = \Bhat_{\zhat(i), \zhat(j)},\, \text{and} \quad \thetatrue_{ij} = \Btrue_{\ztrue(i), \ztrue(j)}\footnote{Note that this is not $B(\ztrue)$}.
    \end{equation*}
    
    It is easy to see that $\cL_{\Edges}(\zhat) = \frac{1}{n^2}\sum_{i,j} \left\|\Edges_{ij} - \thetahat_{ij}\right\|^2$, and that $\cL_{\Edges}(\ztrue) \leq \frac{1}{n^2}\sum_{i,j} \left\|\Edges_{ij} - \thetatrue_{ij}\right\|^2 = \cL_{E,\Btrue}(\ztrue)$. Using \Cref{lem:upper_bound_mappings}, and the fact that $\zols$ is optimal, we have 
    \begin{equation*}
        \cL_{\Edges}(\zhat) \leq \cL_{\Edges}(\zols) + D_n \leq \cL_{\Edges}(\ztrue) + D_n \leq \cL_{E,B^*}(\ztrue) + D_n,
    \end{equation*}
    or equivalently,
    \begin{equation}
        \label{eq:inequality_gap_mse}
        \sum_{i,j} \left\|\Edges_{ij} - \thetahat_{ij}\right\|^2 \leq \sum_{i,j} \left\|\Edges_{ij} - \thetatrue_{ij}\right\|^2 + n^2 D_n.
    \end{equation}

    From there, letting $ \|\Edges-\thetahat\|^2 = \sum_{i,j} \left\|\Edges_{ij} - \thetahat_{ij}\right\|^2 $, we use the sesquilinearity of the inner product to get 
    \begin{equation}
        \label{eq:sesquilinearity}
    \|\Edges-\thetahat\|^2 = \|\thetahat - \thetatrue\|^2 + 2\langle \thetahat - \thetatrue, \thetatrue - \Edges \rangle + \|\Edges - \thetatrue\|^2,
    \end{equation}
    which we can combine with \cref{eq:inequality_gap_mse} to obtain
    \begin{equation*}
        \|\thetahat - \thetatrue\|^2 + 2\langle \thetahat - \thetatrue, \thetatrue - \Edges \rangle + \|\Edges - \thetatrue\|^2 \leq \|\Edges - \thetatrue\|^2 + n^2 D_n,
    \end{equation*}
    and thus
    \begin{equation}
        \label{eq:cross_product_term}
        \frac{1}{n^2}\sum_{i,j}\left\|\thetahat_{ij} - \thetatrue_{ij} \right\|^2 \leq \frac{2}{n^2} \left\langle \thetatrue - \thetahat, \Edges - \thetatrue  \right\rangle + D_n.
    \end{equation}

    We bound the cross product term as in \citet{gao_rateoptimal_2015a}\footnote{None of the concentration inequalities use explicitly the optimality of $\zols$.}, to obtain that for any constant $C'>0$, there exists a constant $\tilde{C}>0$ only depending on $C'$ such that
    $$
         \sum_{i,j}\left\|\thetahat_{ij} - \thetatrue_{ij} \right\|^2  \leq 2\tilde{C}\sqrt{n\log k} \sqrt{\sum_{i,j}\left\|\thetahat_{ij} - \thetatrue_{ij} \right\|^2}  + 4 \tilde{C}^2\sqrt{k^2 + n\log k}\sqrt{n\log k} + D_nn^2,
    $$
    with probability at least $1-\exp(-C'n\log k)$. Solving the above quadratic inequality, we get the following inequality:
    \begin{equation}
    \label{eq:quadratic_form_rate_sbm}        
     \sqrt{\sum_{i,j}\left\|\thetahat_{ij} - \thetatrue_{ij} \right\|^2} \leq \frac{1}{2}\left(2\tilde{C}\sqrt{n \log k} + \sqrt{\left(20\tilde{C}^2\sqrt{k^2+n\log k }\sqrt{n\log k}\right) + 4D_nn^2}\right).
    \end{equation}
    Since from \Cref{prop:optimality_gap} $D_nn^2 = n$, for some constant $C>0$ only depending on $\tilde{C}$ we have
    $$
        \sum_{i,j}\left\|\thetahat_{ij} - \thetatrue_{ij} \right\|^2 \leq C \left(k^2 + n\log k\right) ,
    $$
    which implies \cref{eq:quadratic_form_rate_sbm} with probability at least $1-\exp(-C'n\log k)$, concluding the proof for the SBM case.

    In the graphon case, using a similar argument, we obtain that for any constant $C'>0$, there exists a constant $C_1>0$ only depending on $C'$ such that
    $$
    \sum_{i,j}\left\|\thetahat_{ij} - \thetatrue_{ij} \right\|^2 \leq C_1\left(n^2\left(\frac{1}{k^2}\right)^{\alpha \wedge 1} + n^2D_n + k^2 + n\log(k)\right),
    $$
    with probability at least $1-\exp(-C'n\log k)$. Since $D_nn^2 = C_2 n$ for some $C_2 > 0$, setting $k=n^\beta$ and minimizing the right hand side over $\beta \in (0,1)$  yields:
    $$  
    \sum_{i,j}\left\|\thetahat_{ij} - \thetatrue_{ij} \right\|^2 \leq C\left(n^2 \cdot n^{-\frac{2\alpha}{\alpha+1}}+n{\log n}\right),
    $$
    for some constant $C>0$. This concludes the proof.
\end{proof}

\begin{proof}[Proof of \Cref{cor:sparse}]
    The proof follows the same idea as the proof of \Cref{thm:estimation_1}, but applying the concentration bounds found in \citet[Section 4]{klopp_oracle_2017b}. More precisely, equation (24) of \citet{klopp_oracle_2017b}, where the optimality of the least-squares estimator is used, is transformed to incorporate our optimality gap $D_n$.

    Let $\thetatrue \in [0,1]^{n\times n}$ be the matrix with entries $\thetatrue_{ij}=\rho_n W^*(\xi_i,\xi_j)$ for $i\neq j$ and $\thetatrue_{ii}=0$, and write $E := \Edges - \thetatrue$ for the (centered) noise matrix. Let $\thetahat_{ij} = \Bhat_{\zhat(i), \zhat(j)}$ be our estimator and $\thetaols_{ij} = \Bols_{\zols(i), \zols(j)}$ the unrestricted block-constant least-squares oracle (i.e.\ $\zols \in \argmin_{z\in\cZ_{n,k}}\cL_\Edges(\bproj{z})$). For any $z \in \cZ_{n,k}$ and any symmetric $Q \in \mathbb{R}^{k\times k}_{\rm sym}$, let $\Theta_{z,Q}$ denote the symmetric matrix with off-diagonal entries $Q_{z(i),z(j)}$, and define the family
    \[
       \cT[k] := \bigl\{\Theta_{z,Q}\,:\, z\in \cZ_{n,k},\, Q \in \mathbb{R}^{k\times k}_{\rm sym}\bigr\},
    \]
    of $k$-class block-constant matrices. Let $\Theta_*\in\cT[k]$ be the best Frobenius-norm approximation of $\thetatrue$ in $\cT[k]$.

    From \Cref{lem:upper_bound_mappings} and the optimality of $\zols$, for any $\Theta\in \cT[k]$,
    \[
        \bigl\|\Edges - \thetahat\bigr\|_F^2 \;=\; n^2\,\cL_\Edges(\zhat) \;\leq\; n^2\,\cL_\Edges(\zols) + n^2 D_n \;\leq\; \bigl\|\Edges - \Theta\bigr\|_F^2 + n^2 D_n.
    \]
    Expanding the squared Frobenius norm (similarly to \Cref{eq:sesquilinearity}) yields the analogue of \citet[eq.~(24)]{klopp_oracle_2017b} with our additive optimality gap $D_n$:
    \begin{equation}\label{eq:basic_inequality_klopp_with_Dn}
        \bigl\|\thetahat - \thetatrue\bigr\|_F^2 \;\leq\; \bigl\|\thetatrue - \Theta_*\bigr\|_F^2 + 2\bigl\langle \thetahat - \thetatrue,\, E\bigr\rangle + 2\bigl\langle \thetatrue - \Theta_*,\, E\bigr\rangle + n^2 D_n.
    \end{equation}

    From there, we can conclude exactly as in \citet{klopp_oracle_2017b}.

\end{proof}

\begin{remark}
    We note that in the sparse case, we have an unrestricted least square estimator, i.e.\ we do not constraint the size of the partitions nor the maximum value of the block matrix. \citet{klopp_oracle_2017b} also showed more results for constrained least square estimators. An interesting question is whether we could get such an estimator using a penalized version of the srGWB, where we add a penalty term to the objective that penalizes small partitions. We leave this question for future work.
\end{remark}

\section{Algorithmic details}\label{appendix:algorithmic}
\begin{algorithm}[h]
\caption{Conditional Gradient Solver for srGWB}
\label{alg:SRGWB_pot}
\begin{algorithmic}[1]

\STATE \textbf{Input:} coupling $\pi^{(0)} \in \Pi_{[n],k}$, matrix $B \in [0,1]^{k\times k}$

\FOR{$t = 0,1,2,\dots$ until convergence}

    \STATE 
    \[
    G^{(t)} \leftarrow \nabla_{\pi}\, \mathcal{L}_{\Edges,B}(\pi^{(t)}).
    \]

    \STATE 
    \[
    \psi^{(t)} \leftarrow
    \operatorname*{argmin}_{\psi \in \Pi_{[n],k}}
    \langle \psi, G^{(t)} \rangle .
    \]

    \STATE Choose step size $\gamma^{(t)} \in [0,1]$ by exact line search.

    \STATE 
    \[
    \pi^{(t+1)}
    \leftarrow
    (1-\gamma^{(t)})\pi^{(t)} + \gamma^{(t)} \psi^{(t)}.
    \]

\ENDFOR
\STATE \textbf{Output:} $\pi$
\end{algorithmic}
\end{algorithm}

\subsection{On implicit regularization toward deterministic couplings}\label{appendix:implicit_reg}
\citet{assel2025distributional} observed that srGW implementations on networks often converge to \emph{deterministic} solutions, even under entropic regularization (e.g., Sinkhorn) which is designed to promote diffuse couplings. This can be partly explained by the fact that deterministic assignments are extreme points of the coupling polytope. We formalize this intuition with two results: (i) deterministic mappings always exist that are stationary points of \eqref{eq:relaxed-problem} (\Cref{lem:almostmap}), and (ii) the best deterministic solution incurs at most an $O(1/n)$ loss relative to the optimal probabilistic coupling (\Cref{prop:optimality_gap}). Still, these results do not fully explain why gradient-based methods typically converge to deterministic solutions -- which can arguably be seen as a form of \emph{implicit regularization}, with gradient descent converging to more structured solutions. Empirically, the BCM solver returns near-fully-deterministic couplings at convergence: across our \totalsynthexp synthetic settings, the fraction of nodes with non-trivial probabilistic mass on more than one block is typically zero (see \Cref{appendix:data_analysis}).

\section{Additional numerical experiments}\label{appendix:data_analysis}

\subsection{Simulated data}\label{subsec:sims}
We consider synthetic experiments where graphs are generated from an unknown stochastic block model or continuous graphon. Given an observed network $\Edges$ on $n$ nodes, our goal is to recover a low-dimensional block structure that approximates the underlying graphon using 
the semi-relaxed Gromov-Wasserstein barycenter.
In our experiments, we consider the \totalsynthexp different latent, unknown network structures, specifically \numgraphons continuous graphons with varying smoothness properties, and \numsbms are SBMs, as defined in \Cref{table:ourexp}.

\begin{table}[h!]
\centering
\footnotesize 
\renewcommand{\arraystretch}{1.2} 

\begin{minipage}[t]{0.48\textwidth}
\centering
\begin{tabular}{cl}
\toprule
\textit{ID} & \textit{Graphon Definition} $W(u,v)$ \\ 
\midrule
1. & \( W(u, v) = u \cdot v \) \\ \addlinespace[0.8em]
2. & \( W(u, v) = 0.5(u + v) \) \\ \addlinespace[0.8em]
3. & \( W(u, v) = \exp\left(-\max(u, v)^{0.75}\right) \) \\\addlinespace[0.8em]
4. & \( W(u, v) = |u - v| \) \\\addlinespace[0.8em]
5. & \( W(u, v) = 1 - |u - v| \) \\\addlinespace[0.8em]
6. & \( W(u, v) = 0.5 + 0.5 \cdot \sin(2\pi u) \cdot \sin(2\pi v) \) \\\addlinespace[0.8em]
\bottomrule
\end{tabular}
\end{minipage}
\hfill %
\begin{minipage}[t]{0.48\textwidth}
\centering
\begin{tabular}{ccc}
\toprule
\textit{ID} & $\nu$ & $B$ \\ 
\midrule
1. & \( \begin{bmatrix} 0.5 & 0.5 \end{bmatrix} \) & \( \begin{bmatrix} 0.2 & 0.8 \\ 0.8 & 0.2 \end{bmatrix} \) \\ \addlinespace[1.5em]
2. & \( \begin{bmatrix} 0.5 & 0.5 \end{bmatrix} \) & \( \begin{bmatrix} 0.8 & 0.2 \\ 0.2 & 0.8 \end{bmatrix} \) \\ \addlinespace[1.5em]
3. & \( \begin{bmatrix} 0.3 & 0.7 \end{bmatrix} \) & \( \begin{bmatrix} 0.8 & 0.2 \\ 0.2 & 0.8 \end{bmatrix} \) \\ \addlinespace[1.5em]
4. & \( \begin{bmatrix} 0.2 & 0.5 & 0.3 \end{bmatrix} \) & \( \begin{bmatrix} 0.8 & 0.1 & 0.2 \\ 0.1 & 0.7 & 0.3 \\ 0.2 & 0.3 & 0.5 \end{bmatrix} \) \\ 
\bottomrule
\end{tabular}
\end{minipage}

\caption{Ground truths considered in our experiments: graphons (left) and SBMs (right).}
\label{table:ourexp}
\end{table}

For each ground-truth structure, we sample graphs of increasing size $n = 200,300,\ldots, 10000$ nodes. For each sampled graph, we apply our pipeline to recover a block structure with at most $k$ blocks which seeks to minimize the semi-relaxed GW objective to the observed adjacency matrix. For the continuous graphons, we set  $k \leq \lceil \sqrt{n} \rceil$ blocks and $k \leq 20$ for the SBMs. This effectively manifests a block estimator for the unobserved ground truth. To solve the associated optimization problem, we rely on the block-coordinate minimization algorithm (\Cref{alg:BCM_relaxed_SBM}) with a conditional gradient solver based on the Frank-Wolfe algorithm \cite{vincent-cuaz_semirelaxed_2022,jaggi_revisiting_2013}: see Remark~\ref{rk:algo}. Our implementation is based on the \texttt{pot} library \cite{flamary_pot_2021a}.
Due to the non-convex nature of the objective, the initialization of the probabilistic labeling $\pi$ plays a crucial role in the quality of the solution. We explored various initialization strategies, and found a clustering-based initialization to be the most effective for graphons, while for SBMs a product based is typically more accurate both in estimating connection probabilities and number of blocks; see discussion below.

We performed 20 independent (Monte Carlo) trials for each configuration (ground-truth structure and sample size), and report the MSE achieved by the recovered estimate for each trial in \Cref{fig:sims_graphons,fig:sims_sbms}.

\begin{figure}[h!]
        \centering
        \includegraphics[width=.95\linewidth,trim={1.3cm 0cm 0cm 0cm},clip]{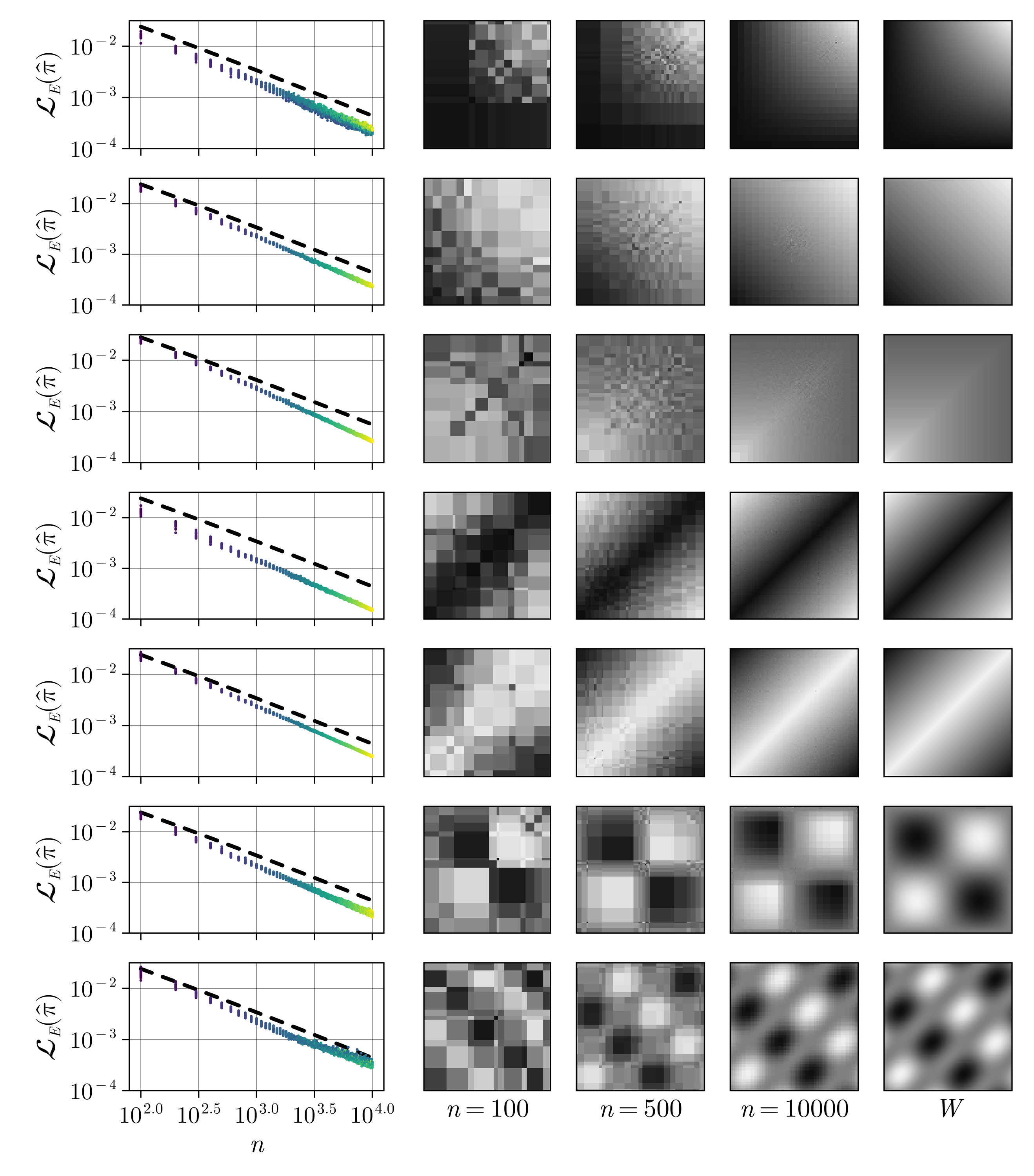}
    \caption{Left: Empirical performance of the relaxed least-squares estimator across different graphon ground-truth structures and sample sizes in 20 independent experiments. We evaluate the Gromov-Wasserstein distance of our estimator to the ground-truth for varying sample sizes $n$. Each subplot represents a different ground-truth structure, with theoretical rates indicated by bold lines. Right: instance of estimated block structure (aligned to the latent truth for ease of visualization) obtained from observed networks of size $n=100,500, 10000$, compared to the ground truth. The ordering reflects that of \Cref{table:ourexp}.}
    \label{fig:sims_graphons}
\end{figure}

\begin{figure}[h!]
        \centering
        \includegraphics[width=.95\linewidth,trim={1.4cm 0cm 0cm 0cm},clip]{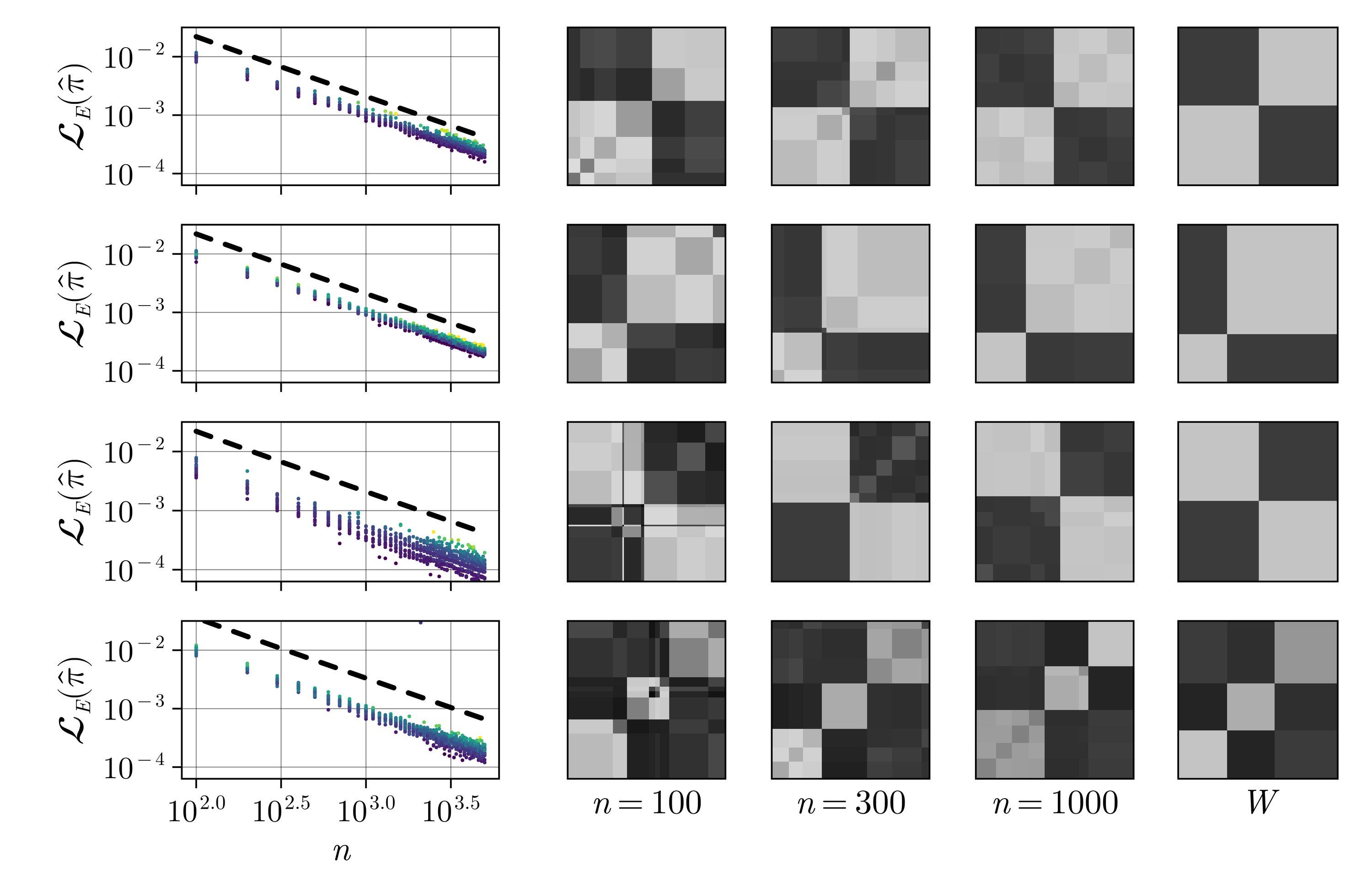}
    \caption{Left: Empirical performance of the relaxed least-squares estimator across different SBM ground-truth structures and sample sizes in 20 independent experiments. We evaluate the Gromov-Wasserstein distance of our estimator to the ground-truth for varying sample sizes $n$. Each subplot represents a different ground-truth structure, with theoretical rates indicated by bold lines. Right: instance of estimated block structure (aligned to the latent truth for ease of visualization) obtained from observed networks of size $n=100,500, 10000$, compared to the ground truth.  The ordering reflects that of \Cref{table:ourexp}.}
    \label{fig:sims_sbms}
\end{figure}

While we set the maximum number of blocks $k$ in our estimate to be controlled by the sample size, e.g., $k \leq \lceil \sqrt{n} \rceil$, the algorithm adaptively determines the effective number of blocks, often selecting fewer than the maximum allowed, as fewer blocks are often optimal for accurate approximation. This is however initialization sensitive: see \Cref{fig:continous_all,fig:sbm_true_k,fig:sbm_sparsity}.

\paragraph{On the choice of initialization.}
    In our experiments, we evaluated the impact of different initialization strategies on the performance of the relaxed least-squares estimator. Specifically, we considered a k-means-based, a spectral-clustering based one, and a product based one.
    The k-means initialization is based on clustering the observed nodes from the observed adjacency matrix to obtain an initial hard assignment of nodes to blocks. This hard assignment is then converted into a probabilistic labeling by setting $\pi_{i,a} = 1/n$ if node $i$ is assigned to block $a$, and $\pi_{i,b} = 0$ for all other blocks. Similarly, the spectral-based initialization clusters the nodes with the usual spectral clustering.
    Differently, the product-based initialization initializes the right-marginal -- encoding the distribution of weights across blocks of the barycenter—rather than the probabilistic labeling itself, which is then initialized as the product measure on the two marginals.
    Our experiments, as illustrated by \Cref{fig:sbm_true_k,fig:sbm_sparsity},  indicate that, for Graphons, clustering-based initialization is significantly less prone to local minimum entrapment compared to the random and product-based strategies. This is evident from the consistently lower mean squared error (MSE) achieved across all ground-truth structures and sample sizes. For graphons, k-means initialization, in particular, effectively starts the optimization process favorably, and achieves the smallest loss.
    However, we observed that the solutions obtained using k-means initialization tend to be less sparse compared to those derived from random or product-based initializations, as appreciated by the colorscale in \Cref{fig:sbm_true_k,fig:sbm_sparsity} displaying the selected number of blocks by the algorithm. This suggests that while k-means initialization mitigates the risk of local minimum, it may favor solutions with higher block number, potentially overestimating the complexity of the underlying structure. In contrast, random and product initializations, despite being more susceptible to local minimum, occasionally yield sparser solutions that align more closely with the true block structure. Yet, when the data is generated from a block model, product-based initialization typically grants the fastest convergence.
\begin{figure}[h!] 
    \centering
\includegraphics[width=1\linewidth]{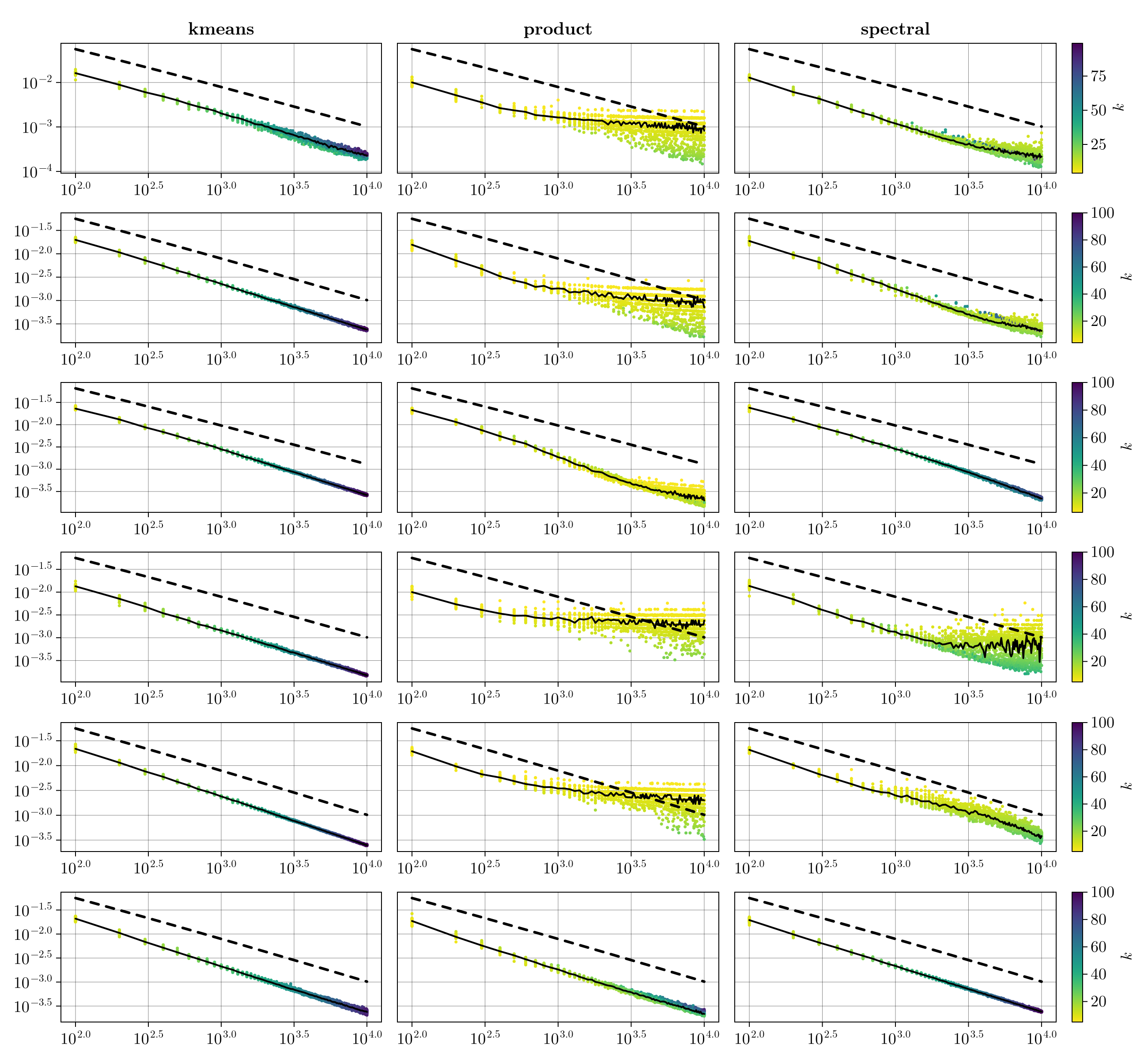}
    \caption{Empirical rates of estimation (GW distance) for graphons for increasing network size. Colorscale shows selected size of best approximating block model, with lighter colors indicating fewer blocks. Full lines represent the median loss and dashed lines represent theoretical rates. Column corresponds to initialization strategy (k-means, product, spectral). The ordering reflects that of \Cref{table:ourexp}.}
    \label{fig:continous_all}

\end{figure}
\begin{figure}[h!]
    \centering
        \includegraphics[width=\linewidth]{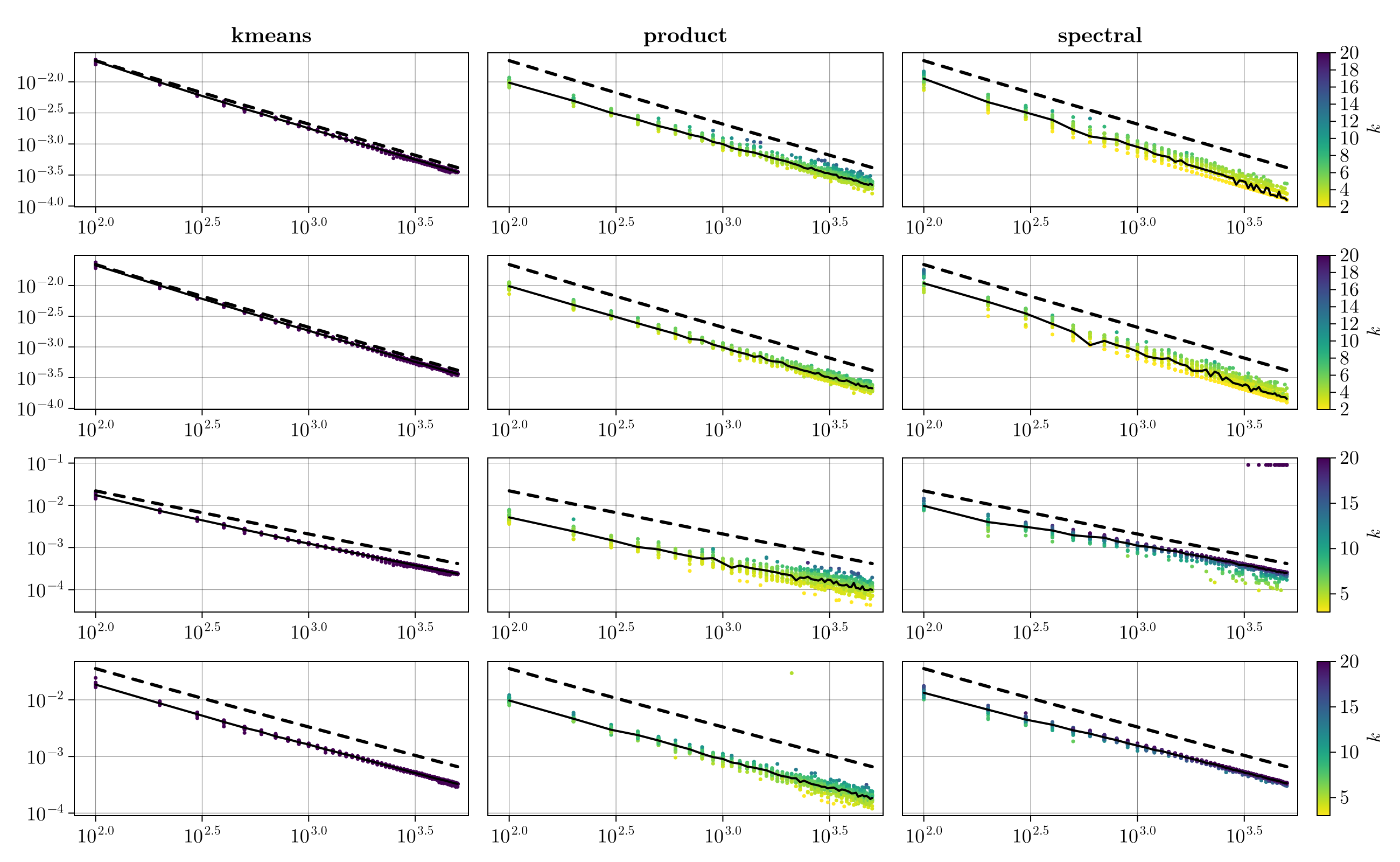}
    \caption{Empirical rates of estimation (GW distance) for SBMs for increasing network size and adaptive block numbers. Colorscale shows selected size of best approximating block model, with lighter colors indicating fewer blocks. Full lines represent median loss and dashed lines theoretical rate.  Column corresponds to initialization strategy (k-means, product, spectral). The ordering reflects that of \Cref{table:ourexp}.}
    \label{fig:sbm_sparsity}
\end{figure}

\begin{figure}[h!]
    \centering
        \includegraphics[width=\linewidth]{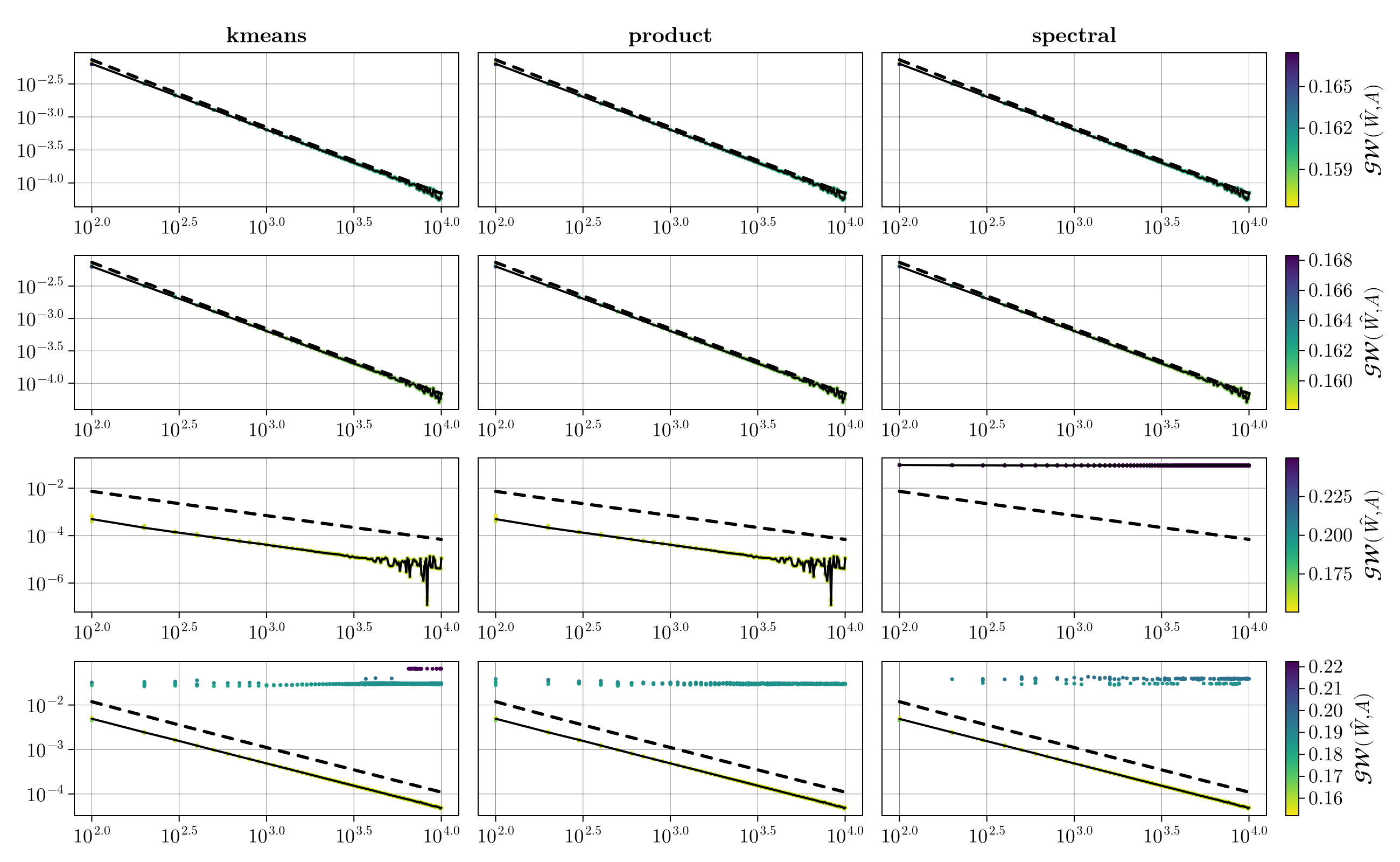}
    \caption{Empirical rates of estimation (GW distance) for SBMs for increasing network size and given (oracle) number of blocks. Full lines represent the median loss and dashed lines represent theoretical rates. The ordering reflects that of \Cref{table:ourexp}.}
    \label{fig:sbm_true_k}
\end{figure}

\clearpage
\paragraph{Algorithmic implementation, computational speed and comparison with existing methods.} \label{rk:algo}
    In the (SRGWB) step of \Cref{alg:BCM_relaxed_SBM}, we employ a revisited Frank-Wolfe method \cite{jaggi_revisiting_2013} to solve the constrained non-convex quadratic optimization problem over the probabilistic labeling \(\pi\). Similarly to \cite{vincent-cuaz_semirelaxed_2022}, we employ a conditional gradient solver which is known to converge to local
    stationary point on non-convex problems \cite{lacoste-julien_convergence_2016}. Since we consider a semi-relaxed problem -- with barycenter marginal estimated in the process, rather than constrained -- the cost of this operation is of order $O(n k)$, with possible parallelization on GPUs. Our implementation relies on the semi-relaxed Gromov-Wasserstein solver made available by the \texttt{POT} library (Python Optimal Transport \cite{flamary_pot_2021a}), and runs entirely on GPUs. We terminate the outer loop when the relative change in $\cL_{\Edges}$ drops below $10^{-6}$; fewer than $50$ outer iterations suffice for $n$ up to $10^4$ (runtime in \Cref{fig:timing}).

    \begin{figure}[h!]
        \centering
        \includegraphics[width=0.7\linewidth]{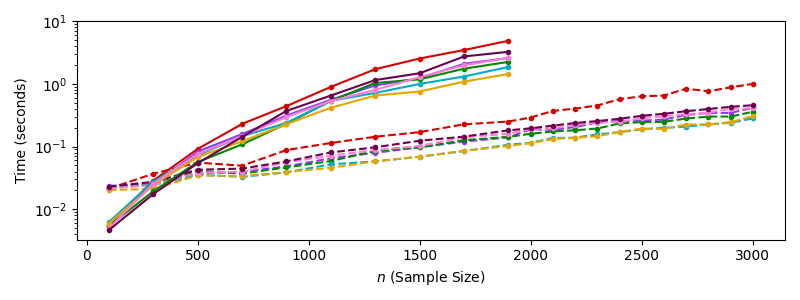}
        \caption{Runtime of our method with k-means initialization for graphons in \Cref{table:ourexp}. 
        Lines represent the time on  CPU (Mac M1) , the dashed lines GPU (Nvidia L40s GPU with a single node).}
        \label{fig:timing}
    \end{figure}
\Cref{fig:benchmark} shows the comparison of our method to other approaches fitting piecewise constant graphons: the histogram estimator with greedy estimation \cite{olhede_network_2014b}, the variational approximation of the Maximum Likelihood Estimator (MLE) \cite{celisse_consistency_2012a} (implemented in \texttt{sparseBM} \cite{frisch_sparsebm_2022}), and the Largest Gap (LG) method \cite{channarond_classification_2012a}. We also benchmark against methods that do not perform dimensionality reduction: the Sort and Smooth (SAS) method \cite{chan_consistent_2014a}, the Neighbourhood Smoothing (NBDsmooth) method \cite{zhang_estimating_2017a}, and the Universal Singular Value Thresholding (USVT) method \cite{chatterjee_matrix_2015a}. We highlight that while all the methods above apply to simple binary graphs, they do not directly extend to more complex network data, contrary to ours.

The srGW-based methods achieve a lower Gromov-Wasserstein distance to the truth while maintaining superior scalability and faster runtimes compared to the HistogramEstimator and \texttt{sparseBM}. These existing block-model methods exhibit significantly higher computational costs as the sample size $n$ increases, often encountering convergence issues (often associated with log-likelihood maximization for binary variable).
While USVT \cite{chatterjee_matrix_2015a} remains competitive in specific regimes (e.g., when the graphon is rank $1$ as is $G1$), it fails to accurately recover dissociative structures, such as those found in SBM 2. Furthermore, unlike our barycentric approach, USVT yields an $n \times n$ estimator that lacks the low-dimensional interpretability required to simplify and compress the underlying generative mechanism.
    \begin{figure}[h!]
        \centering
        \includegraphics[width=\linewidth]{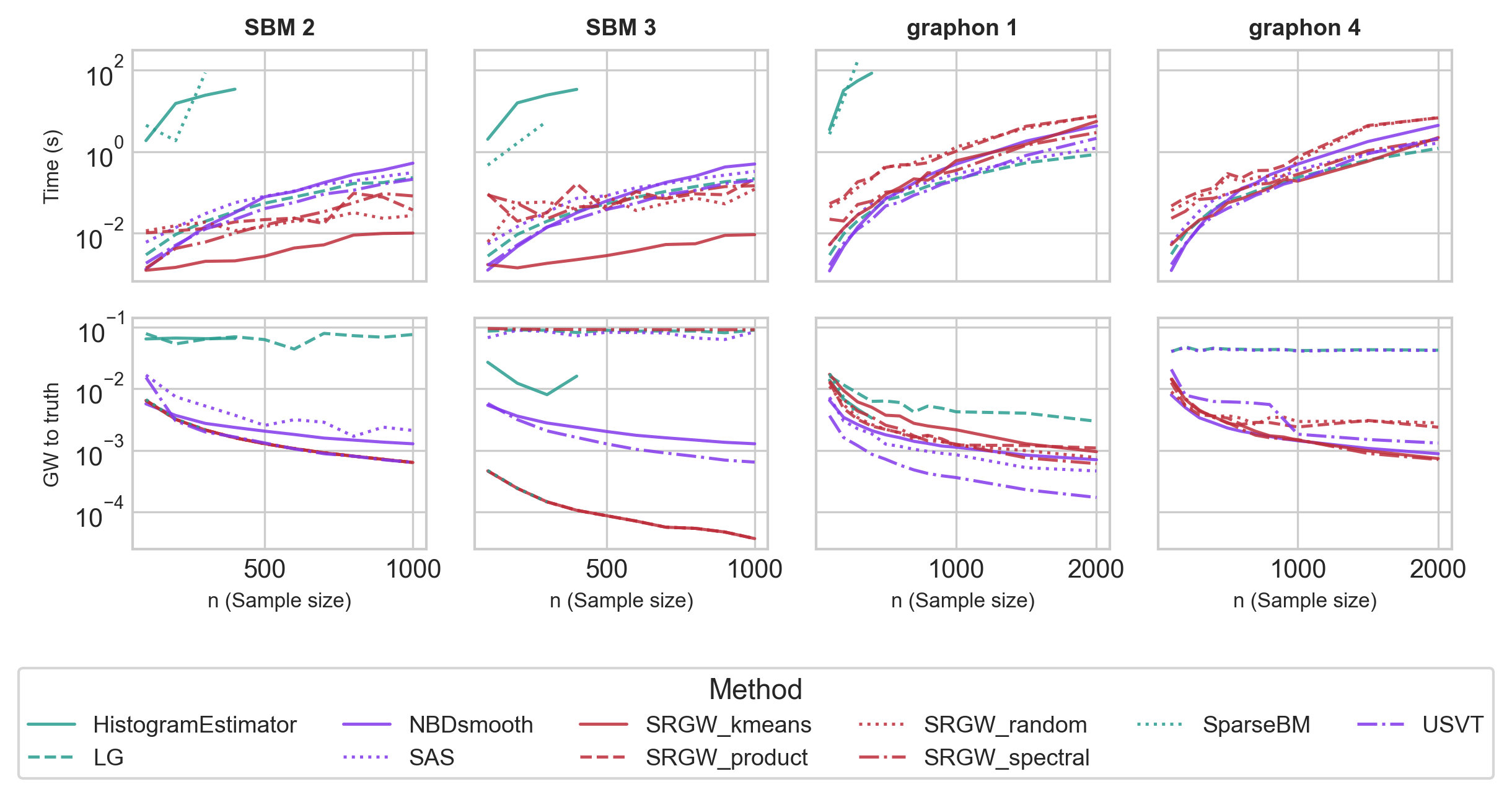}
        \caption{Average running time (seconds, CPU) for different methods. For Histogram and the \texttt{sparseBM}, graph sizes only go up to 300 as sparseBM  would run into numerical instability and convergence issues for higher number of nodes and the computation times were too important. Red indicates our method with different initialization strategies (k-means, spectral, product), blue indicates other low-dimensional methods, and purple indicates methods that do not perform dimensionality reduction.}
        \label{fig:benchmark}
    \end{figure}

\begin{figure}
    \centering
    \includegraphics[width=\linewidth]{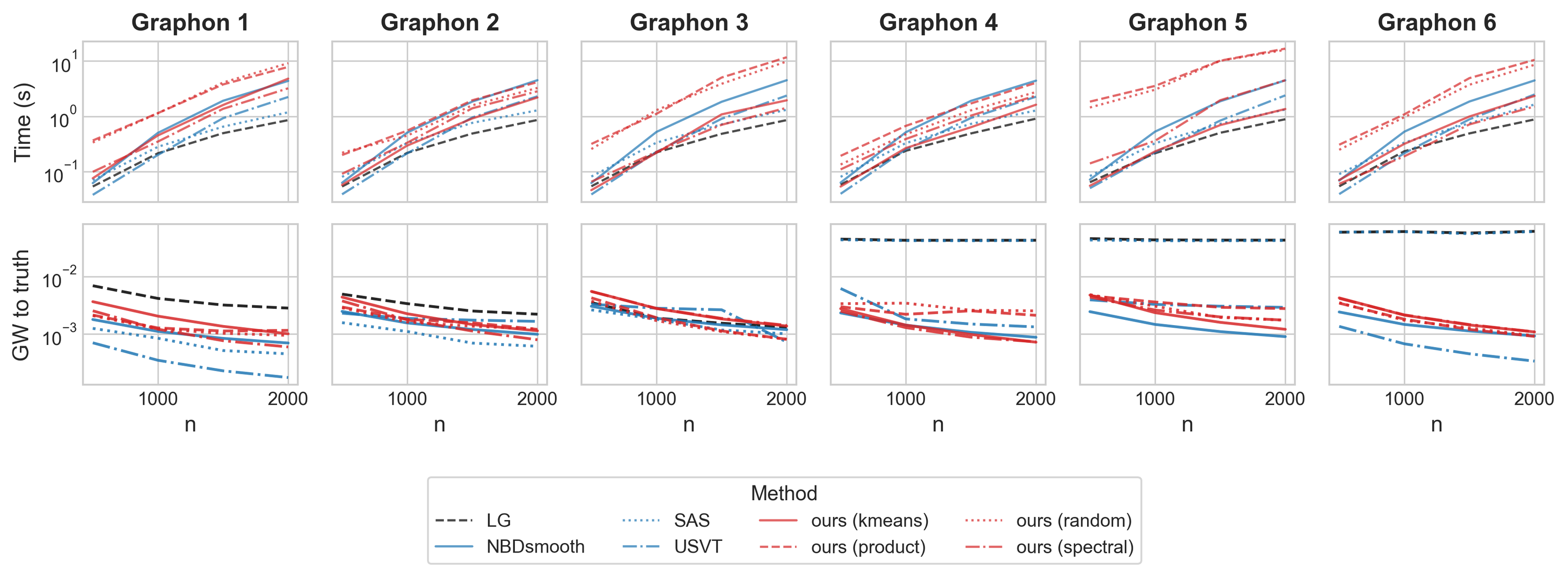}
    \includegraphics[width=0.666\linewidth]{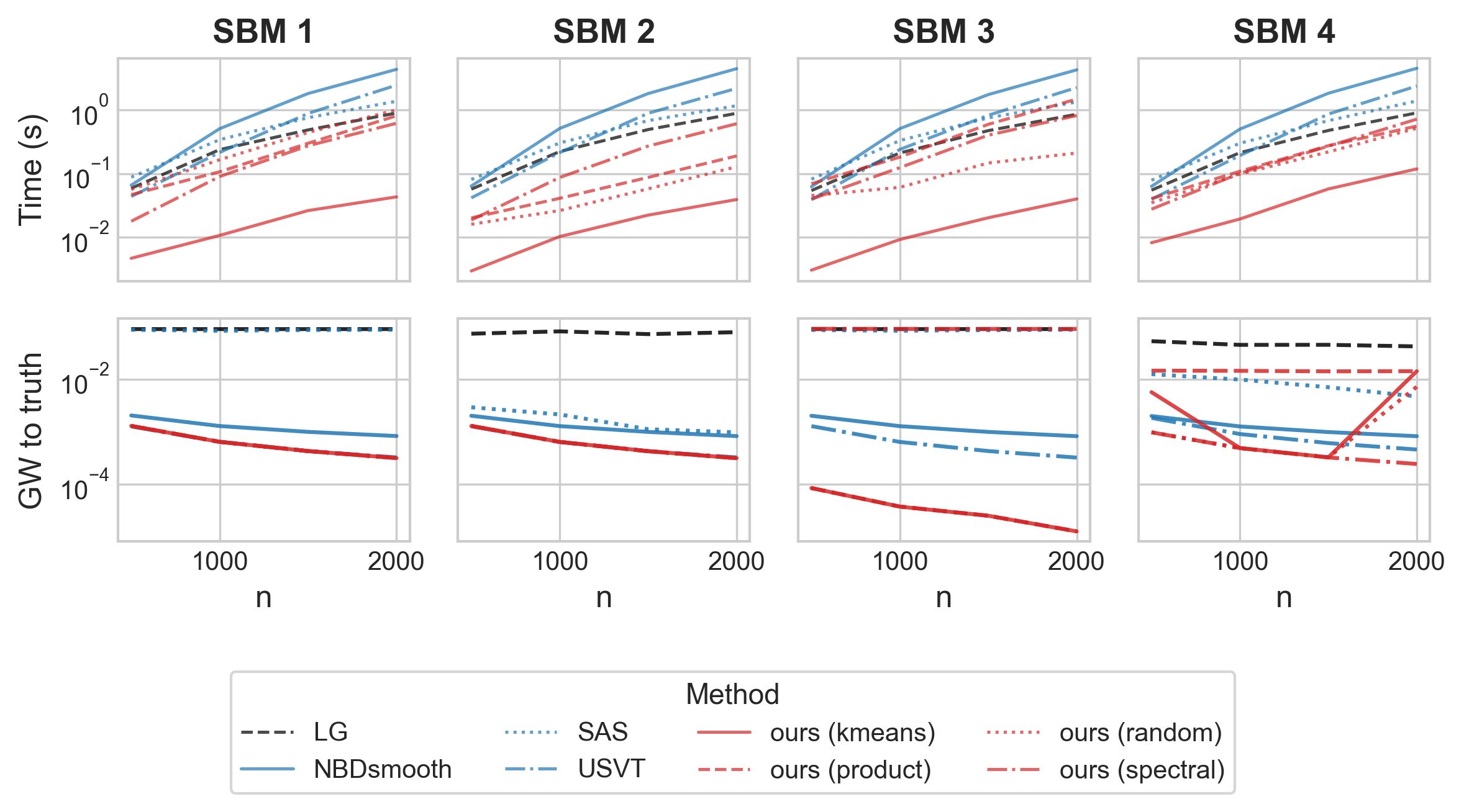}
    \caption{Complete results for all graphons (top) and SBMs (bottom) for all methods.}
\end{figure}
\clearpage

\paragraph{On the relation to structured Gromov--Wasserstein barycenter methods.}
\label{rk:xu_comparison}
A natural question is how our estimator relates to GW-barycenter methods for graphons,
in particular the structured GW barycenter (SGWB) framework of \citet{xu_learning_2021b}.
Although SGWB and our estimator are both based on GW-type objectives, the aims are different: SGWB is designed for graph averaging through regularised barycenters, whereas our objective is recovery of the latent graphon mechanism with statistical guarantees.
We discuss this relation here and explain why a direct head-to-head comparison on the
recovery loss of \Cref{tab:benchmark} would not be informative.

\emph{Theoretical object vs.\ computed object.}
The theoretical analysis in \citet{xu_learning_2021b} relates the cut distance between
graphons to the GW distance between their step-function approximations, leveraging the
weak regularity lemma. The estimator that is actually computed, however, is a
regularised surrogate of this object: the public implementation
solves the inner OT step with a proximal-point Sinkhorn iteration of the form
$\mathrm{kernel}\!\propto\!\exp(-\mathrm{cost}/\beta)$, and adds a quadratic spectral
smoothness penalty $\alpha\,D^\top D$ on the barycenter, with $D$ a discrete
second-difference operator.\footnote{See \texttt{methods/learner.py} of
\url{https://github.com/HongtengXu/SGWB-Graphon}, in particular
\texttt{smoothed\_fgw\_barycenter} and \texttt{proximal\_ot}.}
Neither of these two regularisers is part of the cut-distance argument, and the
diffuseness of the resulting transport plan is governed by the entropic parameter
$\beta$ rather than by the GW objective itself. As a consequence, the rates of
convergence established for the unregularised barycenter do not transfer verbatim to
the implemented estimator, and our \Cref{prop:optimality_gap} does not apply to
$\beta$-regularised solutions, whose support spreads with $\beta$.

\emph{Hyperparameters and absence of an oracle-rate barycenter size.}
The implemented method requires the user to specify $(\alpha,\beta,\gamma)$
controlling the smoothness, entropic and Wasserstein regularisation strengths,
together with two solver budgets and a Sinkhorn tolerance.
We are not aware of a data-driven rule to pick $(\alpha,\beta)$ in our setting, where
ground-truth assignments are unavailable. Furthermore, the barycenter size in the
public implementation is set by a logarithmic heuristic of the largest input graph
($\approx \log_2 n$), which is independent of the smoothness exponent $\alpha$ of the
unknown graphon and therefore does not realise the rate-optimal scaling
$k\asymp n^{1/(\alpha\wedge 1+1)}$ used in \Cref{thm:estimation_1}.

\subsection{Real data: EEG}\label{subsec:realdata_eeg}

\begin{figure}[h!]
    \centering
    \begin{minipage}[c]{0.37\linewidth}
        \centering
        \includegraphics[width=\linewidth]{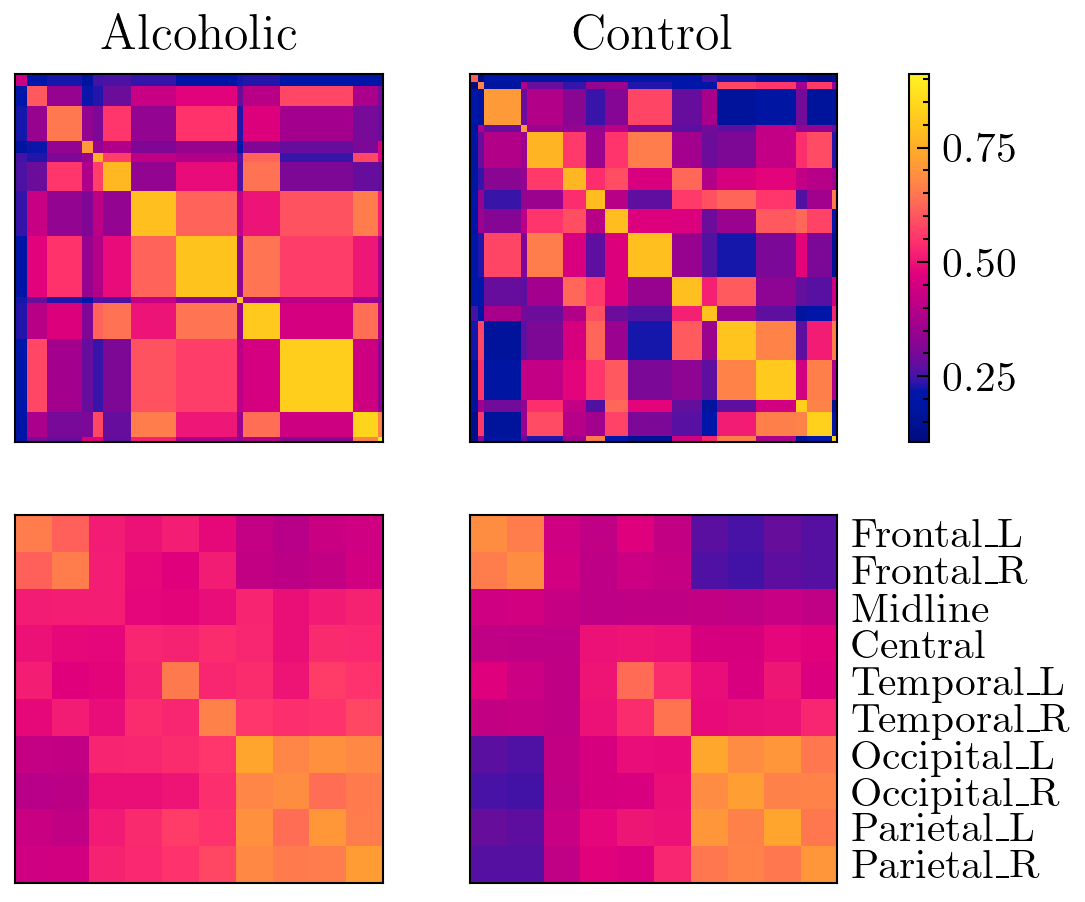}
    \end{minipage}\hspace{0.02\linewidth}
    \begin{minipage}[c]{0.37\linewidth}
        \centering
        \includegraphics[width=\linewidth]{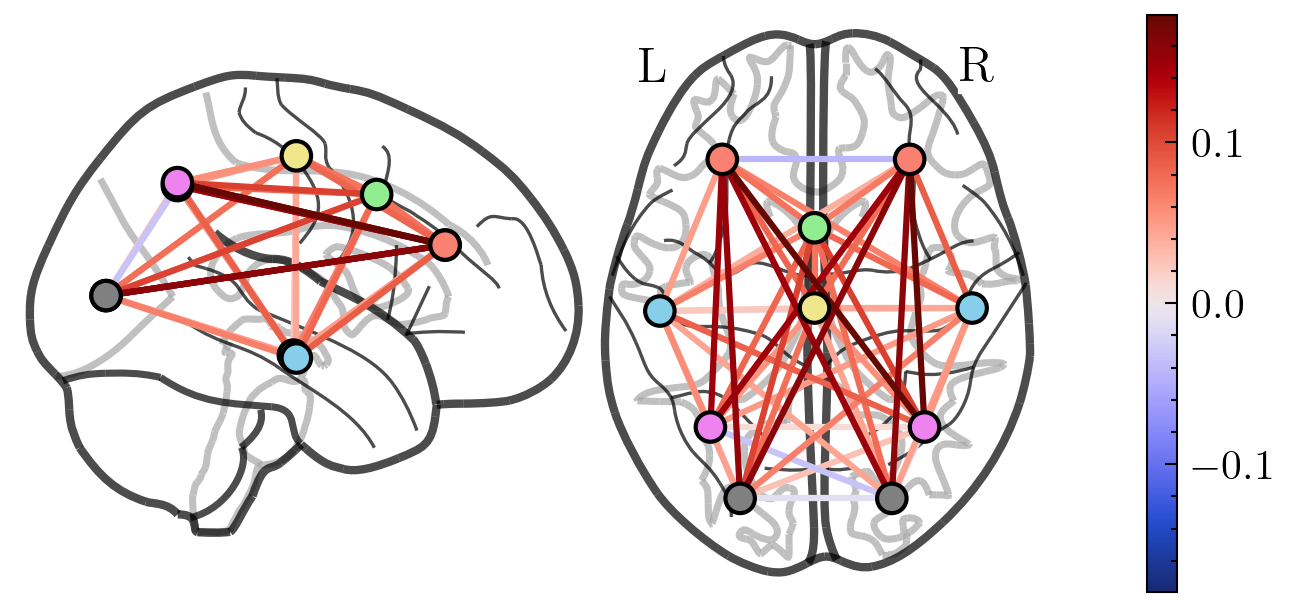}

        \includegraphics[width=\linewidth]{figures/real_data/eeg_groups/eeg_40_task_s1_obj_kernel_linear.png}
    \end{minipage}

    \caption{Low-dimensional co-activation structures estimated by the srGW barycenter for subjects exposed to a single stimulus.}
    \label{fig:layout_s1}
\end{figure}

\begin{figure}[h!]
    \centering
    \begin{minipage}[c]{0.37\linewidth}
        \centering
        \includegraphics[width=\linewidth]{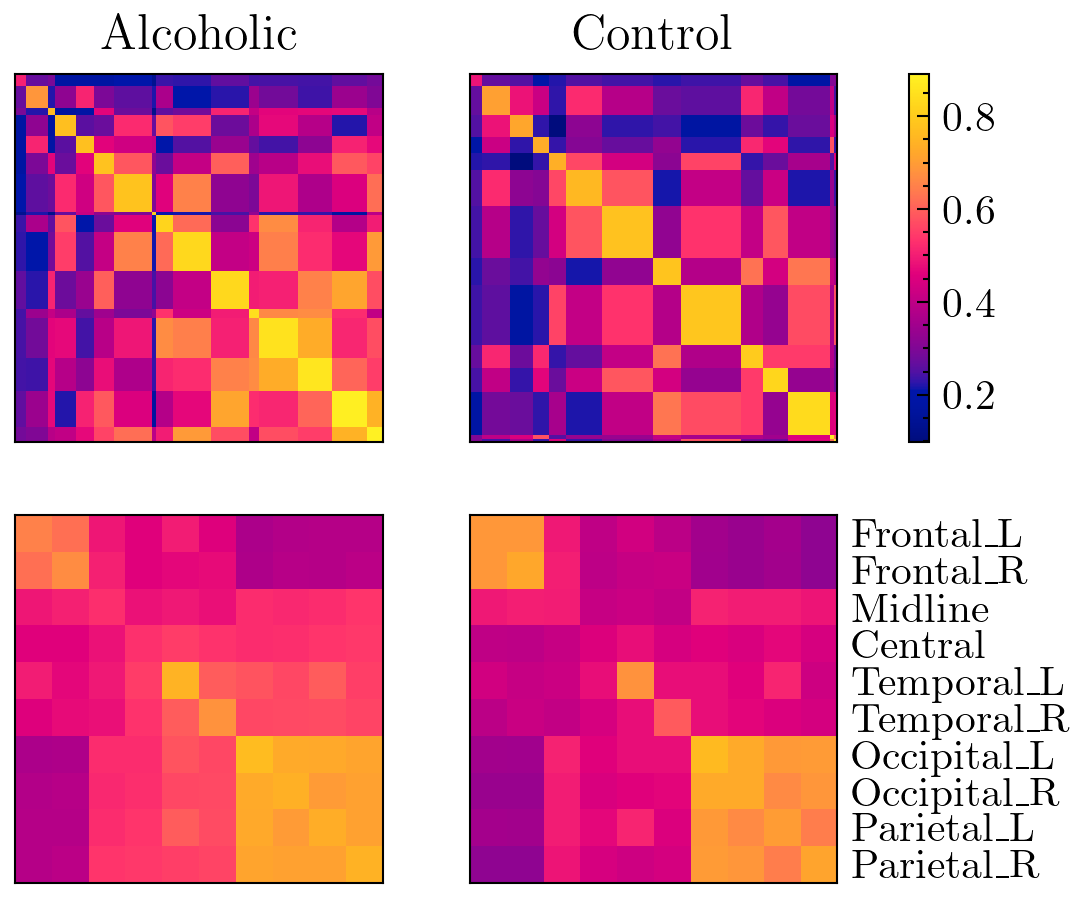}
    \end{minipage}\hspace{0.02\linewidth}
    \begin{minipage}[c]{0.37\linewidth}
        \centering
        \includegraphics[width=\linewidth]{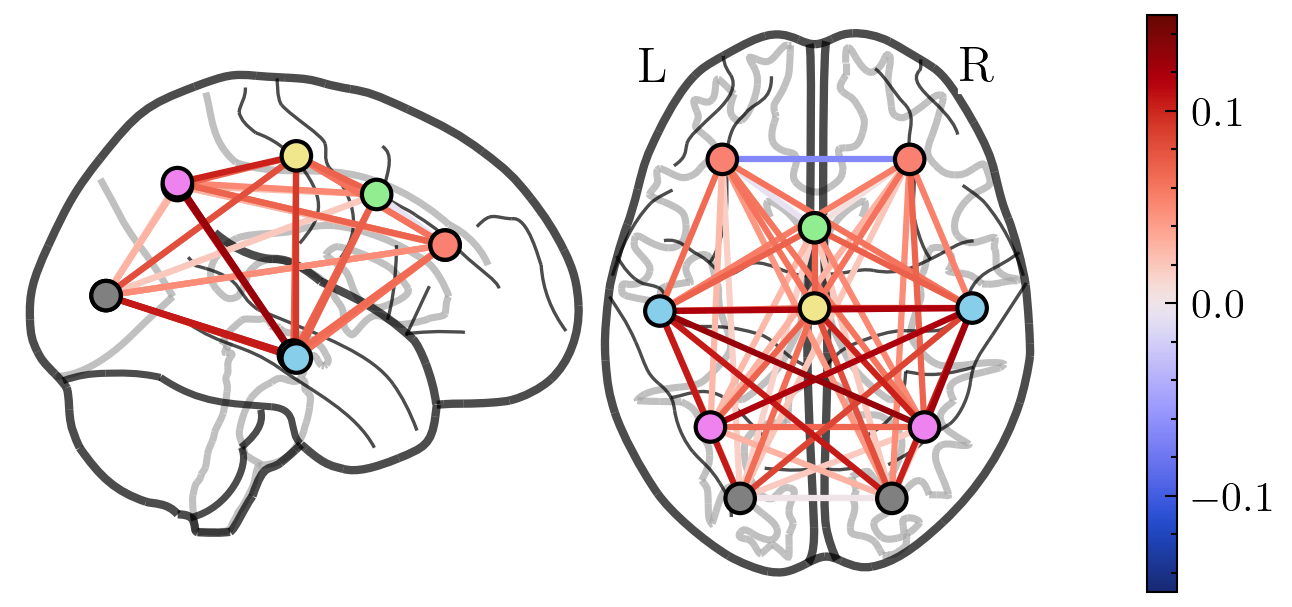}

        \includegraphics[width=\linewidth]{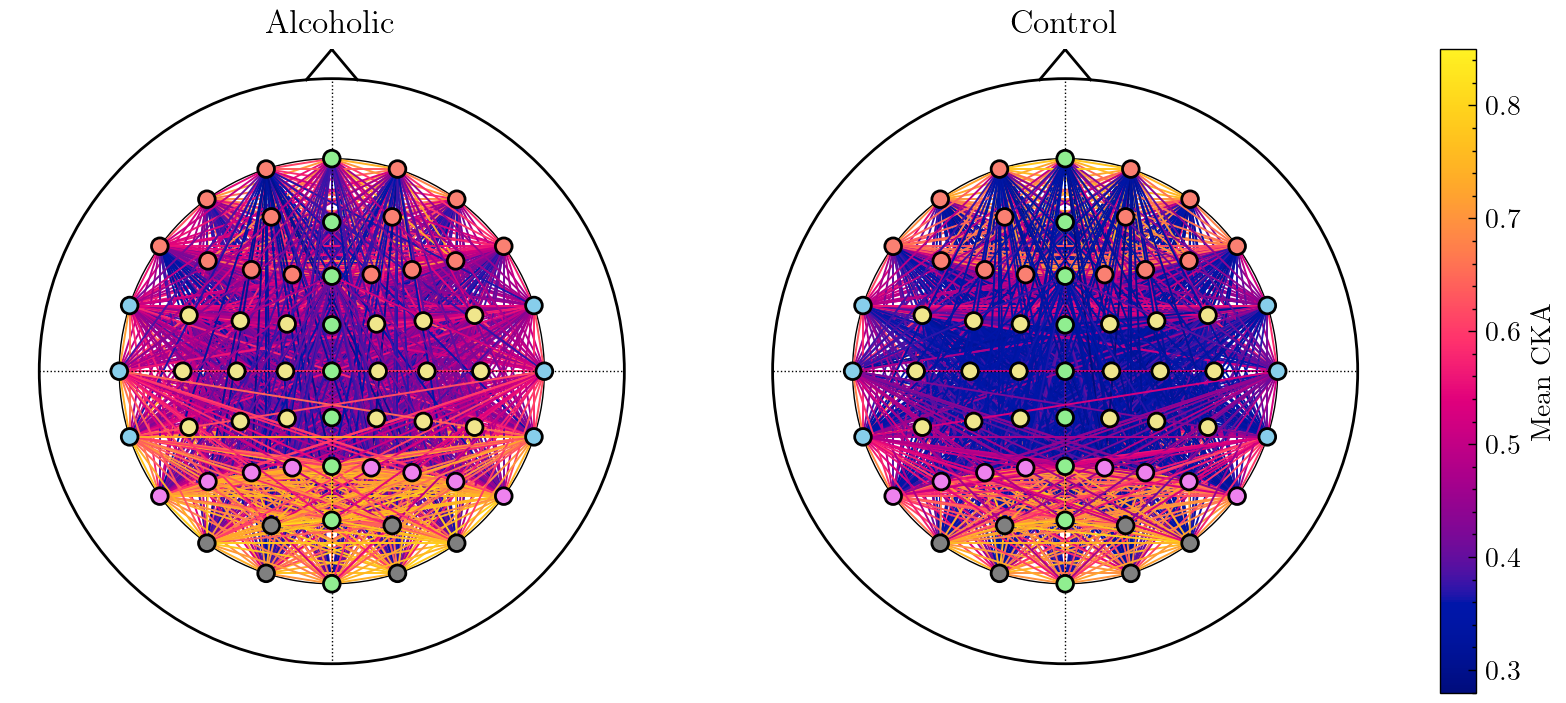}
    \end{minipage}

    \caption{Low-dimensional co-activation structures estimated by the srGW barycenter for subjects exposed to two matching stimuli.}
    \label{fig:layout_s2_match}
\end{figure}

\begin{figure}[h!]
    \centering
    \begin{minipage}[c]{0.37\linewidth}
        \centering
        \includegraphics[width=\linewidth]{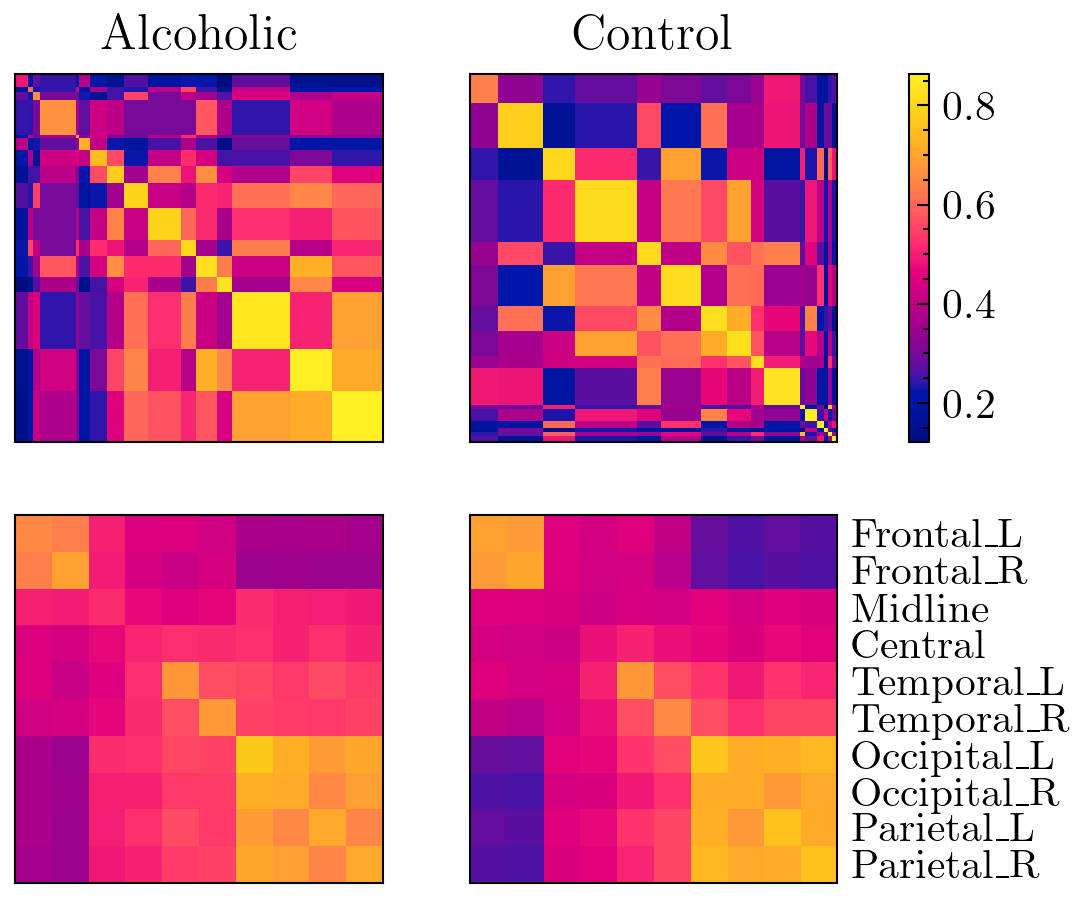}
    \end{minipage}\hspace{0.02\linewidth}
    \begin{minipage}[c]{0.37\linewidth}
        \centering
        \includegraphics[width=\linewidth]{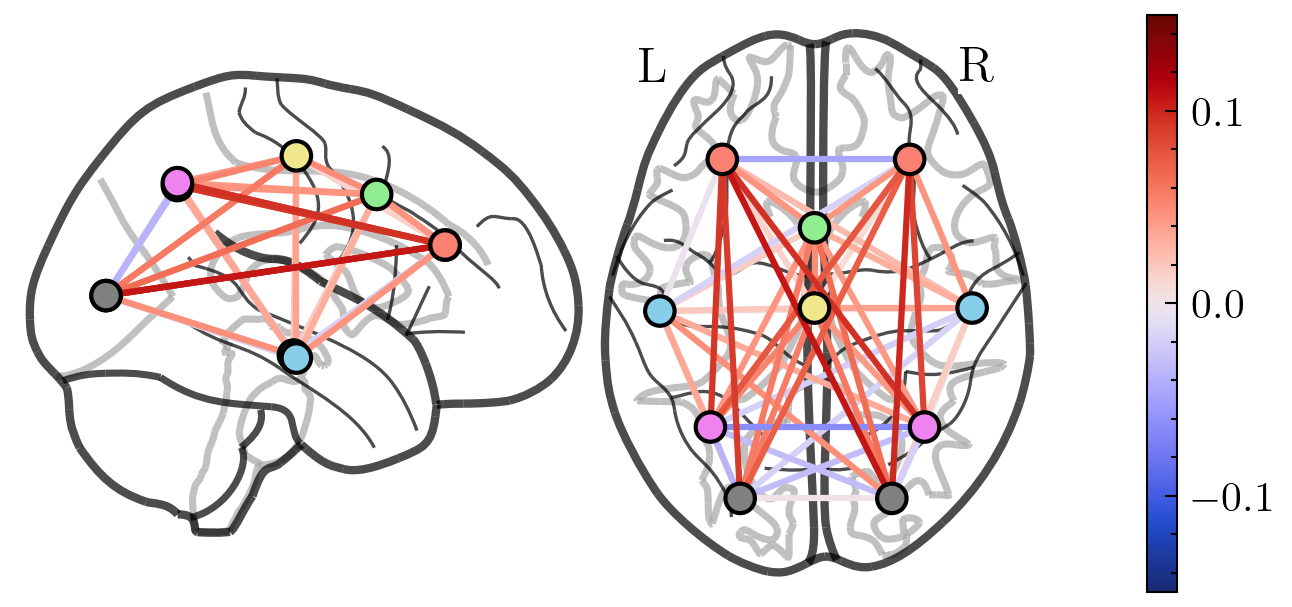}

        \includegraphics[width=\linewidth]{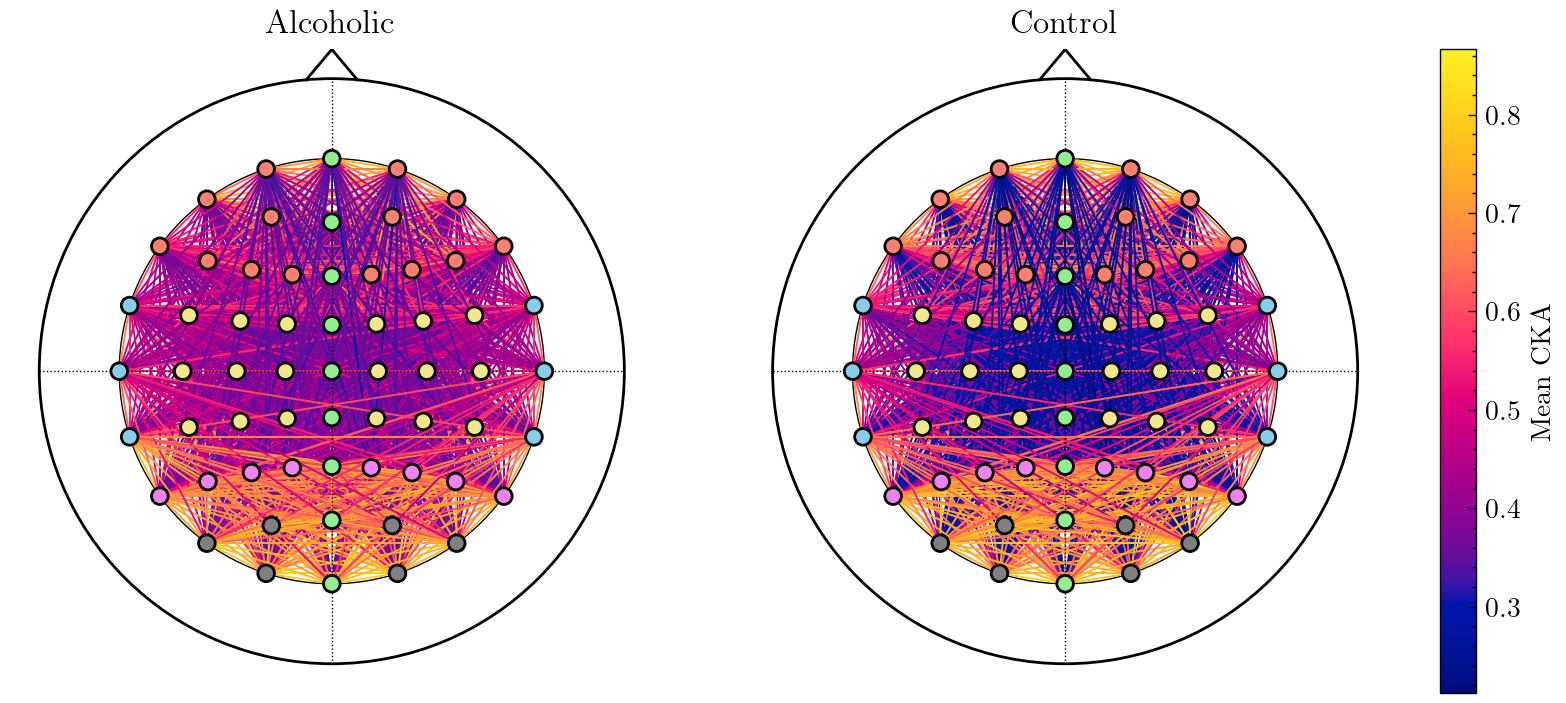}
    \end{minipage}

    \caption{Low-dimensional co-activation structures estimated by the srGW barycenter for subjects exposed to two non-matching stimuli.}
    \label{fig:layout_s2_nomatch}
\end{figure}

\clearpage
\subsection{Real data: OpenFlight Airport Database}\label{subsec:realdata_airport}

\begin{figure}[h!]
    \centering
    \rotatebox{90}{%
    \includegraphics[
            width=1.35\textwidth,
        ]{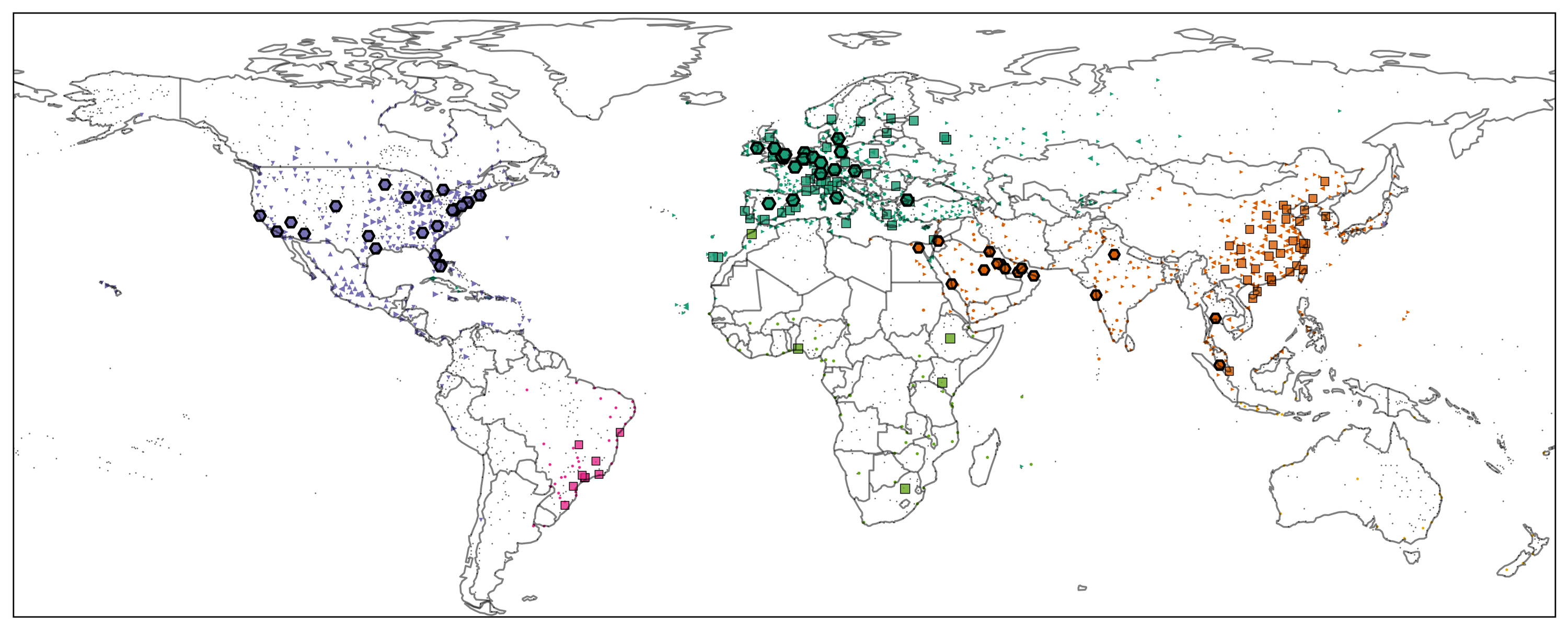}
    }
    \caption{Airports location colored by inferred cluster membership, with node size proportional to cluster connectivity. Our method uses $23$ clusters out of a maximum of $57$. The biggest clusters in terms of number of airports correspond to mostly local airports, and account for close to half of the total number of airports (dark small nodes in the map).}
    \label{fig:airports_world_map_large}
\end{figure}

\end{document}